\documentclass[a4paper,10pt]{article}
\usepackage{amsmath,amsfonts,amssymb,amsthm}
\usepackage{graphicx,subfigure,booktabs,multirow}
\usepackage{mathrsfs,enumerate,color}
\usepackage{algorithm,algpseudocode}
\usepackage[textwidth=14cm]{geometry}
\usepackage[sort&compress, square, numbers]{natbib}
\usepackage{hyperref}

\usepackage{diagbox}
\usepackage{threeparttable}
\usepackage{booktabs}

\newtheorem{theorem}{Theorem}[section]
\newtheorem{lemma}[theorem]{Lemma}

\newtheorem{remark}{Remark}[section]
\newtheorem{definition}{Definition}[section]
\newtheorem{example}{Example}[section]

\newcommand{\iq}{{\bf i}}
\newcommand{\jq}{{\bf j}}
\newcommand{\kq}{{\bf k}}

\newcommand{\Aq}{{\bf A}}
\newcommand{\aq}{{\bf a}}

\newcommand{\Vq}{{\bf V}}
\newcommand{\vq}{{\bf v}}
\newcommand{\Uq}{{\bf U}}

\newcommand{\Lq}{{\bf L}}
\newcommand{\Sq}{{\bf S}}
\newcommand{\Xq}{{\bf X}}
\newcommand{\xq}{{\bf x}}

\newcommand{\Yq}{{\bf Y}}

\newcommand{\wq}{{\bf w}}

\newcommand{\Fq}{{\bf F}}

\newcommand{\Pq}{{\bf P}}

\newcommand{\Qmn}{\mathbb{Q}^{n_1\times n_2}}

\newcommand{\rank}{rank}
\newcommand{\tr}{\operatorname{Tr}}

\newcommand{\diag}{\texttt{diag}}

\title{Non-Local Robust Quaternion Matrix Completion for Color Image and Video Inpainting\thanks{Qiyu Jin, Xile Zhao and Zhigang Jia contribute equally to this work.}}

\author{Zhigang Jia\footnote{Zhigang JIA is with School of Mathematics and Statistics, Jiangsu Normal University,
Xuzhou 221116, P. R. China.
Email: zhgjia@jsnu.edu.cn}, 
~~Qiyu Jin\footnote{Qiyu JIN is with School of Mathematical Science,  Inner Mongolia University, Hohhot  010021, P. R. China. E-mail: qyjin2015@aliyun.com},  
~~Michael K. Ng\footnote{Corresponding author. Department of Mathematics, The  University of Hong Kong, Hong Kong.   E-mail: mng@maths.hku.hk}  
~~and Xile Zhao\footnote{School of Mathematical Sciences\/Research Center for Image and Vision Computing, University of Electronic Science and Technology of China, Chengdu 611731, 
P. R. China. E-mail: xlzhao122003@163.com}}

\begin{document}
\date{}
\maketitle
\begin{abstract}
The  image nonlocal self-similarity (NSS) prior refers to the fact that a local patch often has many
nonlocal similar patches to it across the image and has been widely applied in many recently proposed machining learning algorithms for image processing.  However, there is no theoretical analysis on its working principle in the literature. In this paper, we discover a potential causality between  NSS  and  low-rank property of color images, which is also available to grey images.  A new patch group  based NSS prior scheme is proposed to learn explicit NSS models of natural color images.  The numerical low-rank property of patched matrices is also rigorously proved.  The NSS-based  QMC  algorithm  computes an optimal low-rank approximation to the high-rank color image,  resulting in high PSNR and SSIM measures and particularly the better visual quality. A new tensor NSS-based QMC method is also presented to solve the color video inpainting problem based on quaternion tensor representation.  The numerical experiments on  color images and videos indicate the advantages of  NSS-based QMC over the state-of-the-art methods.\\

{\bf Keywords:}  
Low-rank approximation, quaternion singular value decomposition, nonlocal self-similarity, color image inpainting, color video.
\end{abstract}

\section{Introduction}

The robust quaternion matrix completion (QMC) method  is recently proposed in \cite{jia2018quaternion} to reconstruct low-rank matrices from incomplete and corrupted entries. By encoding color images into purely imaginary quaternion matrices,  this method has been successfully applied to solve the color image inpainting problem with superiority on preserving original color information.  In theory,  the QMC method can successfully recover the original color image which satisfies incoherent conditions and is of low rank  when the missed and noised pixels obey uniform distribution and independence constrains \cite{jia2018quaternion}. However,  we find that the  QMC method performs wonderfully when the target color images are smooth, but often loses the sharp information such as color edges when the original color images contain plenty of textures. Mathematically, the low-rank property of their quaternion matrix representations is the inherent issue.   This motivates us to develop a method to rearrange the pixels of observation such that the rearranged original color image  is of  low-rank.  In this paper,  a new nonlocal self-similarity (NSS) based QMC method is proposed to apply the NSS prior of color images to  significantly improve the completing capability.  Numerical experiments indicate that  this new method can successfully recover  natural color images which are essentially not of low-rank and is efficient to solve the large-scale color image and video inpainting problem.

In the literature, the latest denoising methods have attempted to take advantage of another regularity, the self-similar structures of most images. The nonlocal-means method \cite{buades2005review,buades2005non} seems to be one of the most famous methods to this aim.
 It discovers the property that images are redundant and have self-similar structures and uses this property to find a series self-similar patches or similar pixels. The first version nonlocal-means  uses similar patches to find similar pixels and calculates the weighted means of the similar pixel values. Until  Buades et al. \cite{buades2005review,buades2005non} introduced the idea to find  patches that are similar to the currently processed one, this idea was called nonlocal-means. Since the search for similar pixels will be made in a larger neighborhood, but still locally, the name ``non-local'' is somewhat misleading \cite{buades2011non}, but people have been familiar with it.  Algorithms proposing parsimonious but redundant representations of patches on patch dictionaries
are proposed in \cite{elad2006image,mairal2008sparse,mairal2008learning,yu2010image}.
The BM3D \cite{dabov2007image} is probably the most efficient patch-based current method which creates a 3D block with all patches similar to a reference patch, on which a threshold on the 3D transformed block is applied. Zhang et al. \cite{zhang2010two,gu2017weighted} replaced the  discrete cosine transform (DCT)  by an adaptive local linear transform, the principal component analysis (PCA), and achieves state-of-the-art performance in both quantitative PSNR measures and visual qualities.

 NSS refers to the fact that there are many repeated local patterns across a natural image and those nonlocal similar patches to a given patch can help much the reconstruction of it. The NSS-based image denoising algorithms such as BM3D \cite{dabov2007image},  LSSC \cite{mairal2009}, NCSR \cite{dong2013}, and WNNM\cite{gu2017weighted} have achieved state-of-the-art denoising results. However, it is still open to apply  NSS into color image inpainting, which considers the gross missing and random noise at the same time.   Meanwhile,  QMC  \cite{jia2018quaternion}  exactly reconstructs a color image of sufficiently small rank.  But, this is not true for a color image of high rank. Interestingly enough,  QMC still computes an optimal low-rank approximation to the high-rank color image, although sometime the quality of reconstruction is not very satisfactory.   This motivates us to apply the NSS prior to search similar patches of a single color image and gather them together to generate a commonly  low-rank  quaternion matrix  to be reconstructed by QMC.

The main contribution of this paper is in three aspects:
\begin{itemize}
\item   The working principle of NSS-based approaches is rigorously explained by introducing    numerical rank to quaternion matrices. The  NSS prior based on quaternion representation is first applied to  solve the  color image inpainting problem.   A new NSS-based   robust quaternion matrix completion method (NSS-based QMC) is proposed and is theoretically and experimentally proved efficient to reconstruct  color images from incomplete and corrupted entries.
\item  A novel tensor  NSS-based QMC method (TNSS-based QMC) is developed to solve the color video inpainting problem. Color videos are represented by quaternion tensors. TNSS-based QMC is in fact a novel quaternion tensor completion method.   It is feasible to recover  color  videos  from incomplete and corrupted tubes  (that is,  frontal slices miss pixels at the same positions).
\item The NSS-based and TNSS-based QMC methods are applied to inpaint natural color images and videos from incomplete and corrupted pixels and their performances  are  superior to the state-of-the-art methods on recovering the color and geometric properties such as color edges and textures.
\end{itemize}

This paper is organized as follows.
 In Section \ref{sec:pre}, we  briefly present necessary  information about quaternion matrices and recall the QMC model and theory.
 In Section \ref{sec:main}, we  propose a new NSS-based QMC method, including in the numerical rank analysis, the nonlocal self-similarity, and   a new NSS-based QMC algorithm.
 In Section \ref{sec:large}, we develop a new TNSS-based QMC method to solve the large-scale color image and video inpainting problems.
 In Section \ref{sec:ex}, we apply the NSS-based and TNSS-based QMC methods to practical inpainting problems of color images and videos.
 Finally, the concluding remark and further work are given in Section \ref{sec:con}.

\section{Preliminaries}\label{sec:pre}
In this section, we firstly recall the basic information of quaternion matrices and the robust quaternion matrix completion method \cite{jia2018quaternion}.

\subsection{Quaternion matrices}
Quaternion,  introduced by Hamilton \cite{Hamilton1866},
has one real part and three imaginary parts given by
$$
\aq = a_r + a_i {\bf i} + a_j {\bf j} + a_k {\bf k}
$$
where $a_r, a_i, a_j, a_k \in \mathbb{R}$ and ${\bf i}$, ${\bf j}$
and ${\bf k}$ are three imaginary units satisfying
$$
{\bf i}^2 = {\bf j}^2 = {\bf k}^2 = -1,\quad \iq\jq\kq=-1.
$$
Here a symbol of boldface indicates that it is a quaternion scalar, vector or matrix.
When $a_r = 0$, ${\bf a}$ is called a purely imaginary quaternion.
For simplicity, we denote $\mathbb{Q}$ by a set of quaternions. The set of $n_1\times n_2$ quaternion matrices are defined by
$$\mathbb{Q}^{n_1\times n_2}=\left\{\Aq=A_r+A_i\iq+A_j\jq+A_k\kq\right\}$$
in which $A_r,A_i,A_j,A_k\in\mathbb{R}^{n_1\times n_2}.$   $\Aq\in\mathbb{Q}^{n_1\times n_2}$ is called a purely imaginary quaternion matrix if its real part is zero ($A_r=0$).
The maximum
number of right linearly independent columns of a quaternion matrix
${\bf A} \in \mathbb{Q}^{n_1\times n_2}$ is called the rank of ${\bf A}$, denoted by $\rank(\Aq)$.
For any quaternion matrix $\Aq=[\aq_{ij}]\in\Qmn$,  we define
$$\texttt{absQ}(\Aq):=[|\aq_{ij}|]$$
 and
 $$\texttt{signQ}(\aq_{ij}):=\left\{\begin{array}{ll}
 \aq_{ij}/|\aq_{ij}|, & \text{if}~~|\aq_{ij}|\ne 0;\\
 0, &\text{otherwise.}
 \end{array}\right.$$
For any $\tau>0$,  the shrinkage of quaternion matrix $\Aq$ is defined by
$$ \texttt{shinkQ}(\Aq,\tau):=[\texttt{signQ}(\aq_{ij})\max(\texttt{absQ}(\aq_{ij})-\tau,0)]. $$
Suppose that  quaternion matrix $\Aq\in\mathbb{Q}^{n_1\times n_2}$ (with $n_1\ge n_2$) has the singular value decomposition
\begin{equation}\label{e:qsvd}
{\bf A} = {\bf U}  \Sigma {\bf V}^H,
\end{equation}
where   $$ \Sigma=\begin{bmatrix}
\sigma_1 &0&\cdots&0\\
0& \sigma_2&\cdots&0\\
\vdots&\vdots&\ddots&\vdots\\
0&0&\cdots&\sigma_{n_2}\\
\vdots&\vdots&\ddots&\vdots\\
0&0&0\cdots&0
\end{bmatrix}:=\texttt{diag}(\sigma_1,\cdots,\sigma_{n_2})$$ is a  diagonal (rectangle) matrix with $\sigma_1\ge\sigma_2\ge\cdots\ge\sigma_{n_2}\ge 0$,
and  ${\bf U} \in\mathbb{Q}^{n_1\times n_1}$ and  ${\bf V} \in\mathbb{Q}^{n_2\times n_2}$ are two unitary quaternion matrices \cite{Zhang1997}.
\begin{definition}[\bf Numerical Rank]\label{d:numerrank}
Let $\delta>0$.  A quaternion matrix $\Xq$ is called of $\delta$-rank $r$ if it has $r$ singular values bigger than $\delta$.
\end{definition}
 For any $\tau>0$, we define a low-rank approximation of $\Aq$ as
$$\texttt{approxQ}(A,\tau)=\Uq\texttt{diag}(\sigma_1,\cdots,\sigma_k,0,\cdots,0)\Vq^H,$$
where $\sigma_1\ge\cdots\ge\sigma_k>\tau$ and  the rest singular values of $\Aq$ that are smaller or equal to $\tau$ are replaced with zeros.

The inner product between two quaternion matrices
${\bf A}$ and ${\bf B}$ is defined by
\begin{equation}\left \langle
{\bf A}, {\bf B} \right \rangle = \tr( {\bf A}^H {\bf B}
),\label{new1}
\end{equation} where $\tr( {\bf A}^H {\bf B} )$ denotes the
trace of  $ {\bf A}^H {\bf B}.$ The matrix norms of $\Aq$ are defined by
the $\ell_1$ norm  $\left\Vert \Aq\right\Vert _{1} :=\sum\limits_{i=1}^{n_1}\sum\limits_{j=1}^{n_2}\left\vert
\aq_{ij}\right\vert$; the $\infty$-norm
$\left\Vert \Aq\right\Vert _{\infty } :=\max\limits_{i,j}\left\vert \aq_{ij}\right\vert$;
the F-norm $\left\Vert \Aq\right\Vert _{F}=\sqrt{\sum\limits_{i=1}^{n_1}\sum\limits_{j=1}^{n_2}\left\vert \aq_{ij}\right\vert ^{2}} :=\sqrt{ \tr\left( \Aq^H\Aq\right) }$;
the spectral norm $\|\Aq\|_2 :=\max \{\sigma_1,\cdots, \sigma_r\}$
and the nuclear norm
$\|{\bf A} \|_{*} := \sum\limits_{i=1}^r \sigma_i$, where $\sigma_1,\cdots, \sigma_r$ are all nonzero singular values of $\Aq$.

\subsection{The robust quaternion matrix completion method}
The  QMC method, recently proposed in \cite{jia2018quaternion},  aims to recover an $n_1$-by-$n_2$ quaternion matrix $\Lq_{0}$ from its partly known entries with noise.
Let  $\mathcal{P}_{\Omega}$ be the orthogonal projection onto the linear space of matrices supported on $\Omega\subseteq [1:n_{1}]\times [1:n_{2}]$,
$$
\mathcal{P}_{\Omega} ({\bf X}) = \left \{ \begin{array}{cc}
 {\bf x}_{i,j},& (i,j)\in \Omega,  \\
  0,  &(i,j)\notin \Omega.
\end{array}
\right.
$$
A few available data of $\Lq_{0}+\Sq_{0}$ in the subset $\Omega$
are written as
$
\Xq=\mathcal{P}_{\Omega}(\Lq_{0}+\Sq_{0}),
$
where  $\Sq_0$ is a sparse matrix containing the noise information. For each observed quaternion matrix $\Xq\in \Qmn$,  we consider the following robust  quaternion matrix completion model as in \cite{jia2018quaternion}  to recover the underlying quaternion matrix,
\begin{equation} \label{model}
\begin{array}{rl}
\underset{\bf L, \bf S}{\min}&
\left\Vert {\bf L} \right \Vert_{\ast} + \lambda \left\Vert {\bf S}\right\Vert _{1}\\
\ {\rm s.t. } &
\mathcal{P}_{\Omega}\left({\bf L}+{\bf S}\right)={\bf X}.
\end{array}
\end{equation}
From  \cite[Theorem 2]{jia2018quaternion}, the solution $\hat{\bf L}$ of (\ref{model}) with $\lambda=\frac{1}{\sqrt{\rho n_{(1)}}}$
 is exact, i.e.,
$\hat{\Lq}=\Lq_{0}$, provided that the rank of $\Lq_{0}$ is under a given bound and $\Sq_{0}$ is sufficiently sparse. Here, $\rho=|\Omega|/n_1n_2$ denotes the ratio of sampling and $n_{(1)}=\max\{n_{1},n_{2}\}$.
 Note that the optimal  balance parameter $\lambda$ is fixed, which sets us free from the hard work on parameter selection.  The prior assumption on the rank of $\Lq_0$ is one of motivations to write this paper.

The robust quaternion matrix completion can be applied into color image inpainting.
Essentially, the observed color image with the spatial resolution of $n_1 \times n_2$ pixels
is represented by an $n_1 \times n_2$ quaternion matrix ${\bf X}$ in ${\mathbb Q}^{n_1 \times n_2}$ as follows:
$ \Xq_{ij}
 = R_{ij} {\bf i} + G_{ij} {\bf j} + B_{ij} {\bf k}, \quad
1 \le i \le n_1, 1 \le j \le n_2,
$
where $R_{ij}$, $G_{ij}$ and $B_{ij}$ are the red, green and blue
pixel values respectively at the location $(i,j)$ in the image.
From Theorem 2 in \cite{jia2018quaternion}, the rank of the targeted color image (denoted by $\Lq_{0}$)  is smaller,  the possibility of exactly recovering it is higher.  However, the natural color images are not always of low rank. When the original $\Lq_0$ is of high rank,  the computed $\hat{\Lq}$ by the soft threshold method suggested in \cite{jia2018quaternion}  is not the exact solution anymore but  an optimal low-rank approximation of $\Lq_0$.  This undoubtedly reduce our expectation on reconstructions from observed color images with missing data.
This foundation motivates us to enhance the QMC method with the patching or non-local mean technique.

\section{Non-Local Robust Quaternion Matrix Completion} \label{sec:main}
In this section, we present a new non-local QMC approach and a fast patch-based QMC algorithm based on the ADMM frameworks.

\subsection{Model and Low-Rank Theory}
Now we  propose a NSS-based  framework for robust color images completion under quaternion representation,  called by NSS-based QMC.

Suppose an observed color image is expressed as a quaternion matrix ${\bf X}= R {\bf i} + G {\bf j} + B {\bf k}\in \mathbb{Q}^{  n_1 \times n_2}$,
where $R$, $G$ and $B\in \mathbb{R}^{  n_1 \times n_2}$ represent the red, green and blue
 components, respectively.
By moving  across the spatial domain with overlaps, we  build a group of  patches
\begin{equation}\label{e:Gpatches} \mathcal{G}=\{\Yq_{i,j}\in \mathbb{Q}^{ w \times h}\}\end{equation}
in which  $\Yq_{i,j}$ is a patch centred at $(i,j)$ of  color image $\Xq $ and
 $w\times h$ is the fixed patch size, $0<w\le n_1$ and $0<h\le n_2$.
Using a unit metric (say, the Euclidean distance with a threshold), the nonlocal similar patches are found from $\mathcal{G}$ and gathered into different subgroups.
The classic method to do so is  firstly setting $m$ key patches (covering the whole color image), and then finding a fixed number of  patches most similar to it by  block matching.
Mathematically,  the subgroup of similar patches to each given key patch $\Yq_{i,j}$ is
\begin{equation}\label{e:grouppatch}\mathcal{G}_{i,j} =\left\{ \Yq_{s,t}: f(\Yq_{s,t},\Yq_{i,j}) \leq \delta \right\}\!,\end{equation}
where $f$ is a distance  function between two patches, e.g.,
$f(\Yq_{s,t},\Yq_{i,j})$ $=\|\Yq_{s,t}-\Yq_{i,j}\|_F$,  and $\delta>0$ is a manual threshold.
It is practically convenient to gather a fixed number of similar patches in each subgroup. That is, we search a priorly decided number of similar patches, i.e., $|\mathcal{G}_{i, j}|=d$ with $d$ being a  positive integer.
Finally,  one quaternion matrix $\Xq_{i,j}\in \mathbb{Q}^{  (wh) \times d }$ is constructed  from each subgroup of patches, $\mathcal{G}_{i,j}$,  with ordering all elements lexicographically as a column vector, i.e,
\begin{equation}\label{e:Xk}
\Xq_{i,j}=[{\tt vec}(\Yq_{i_1,j_1}),\cdots,{\tt vec}(\Yq_{i_d,j_d})],
\end{equation}
where $\Yq_{i_k,j_k}$ denotes  the $k$-th element of $\mathcal{G}_{i,j}$. Here,   ${\tt vec}(\Yq_{i,j})$ means to stack all columns of quaternion matrix $\Yq_{i,j}$ into a quaternion vector.
The  diagram of  the construction is displayed in Fig.\ref{Fig:patchedidea}.
 It is interesting to observe that the singular values of $\Xq_{i,j}$ decay fast.
\begin{figure}
\begin{center}
\includegraphics[height=0.55\linewidth,width=1.0\linewidth]{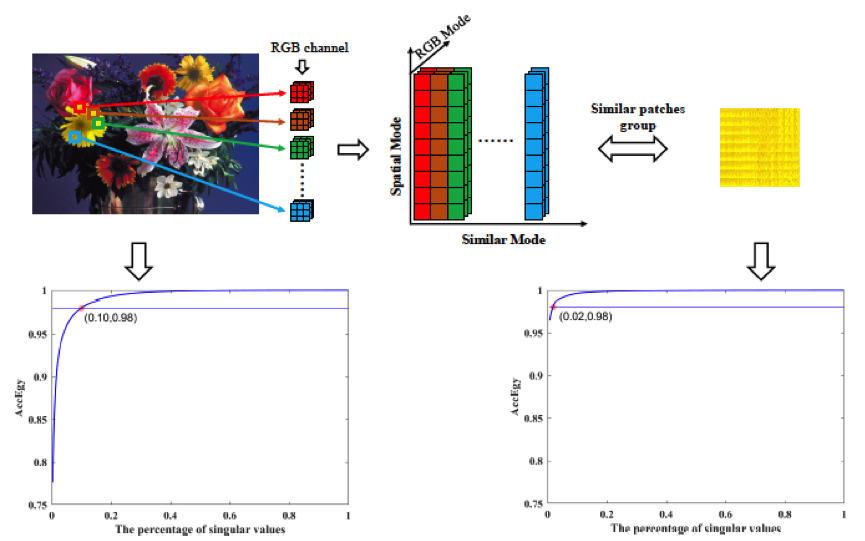}
\end{center}
\vspace{-0.6cm}
\caption{{\protect\small {The operating principle of the NSS-based QMC.  The  bottom two
graphs  are the accumulation energy ratios (AccEgy$=\sum_{i=1}^{k}\sigma_i^2/\sum_j\sigma_j^2$, where $\sigma_i$ denotes the $i$th singular value) of original color image $\Xq$ and  the patched matrix  $\Xq_{i,j}$.}}}
\label{Fig:patchedidea}
\end{figure}

Next, we apply the QMC algorithm (see Algorithm 1 in Section \ref{ss:admm}) on each     $\Xq_{i,j}$ defined by \eqref{e:Xk} to  recover the subgroup $\mathcal{G}_{i,j}$.   The aim is to  compute a low-rank reconstruction $\widehat{\Lq}_{i,j}$ by applying the QMC method on $\Xq_{i,j}$.  In the converse way of generating $\Xq_{i,j}$, we get the low-rank approximations to  patches from $\mathcal{G}_{i,j}$,  denoted by
 \begin{equation}\label{e:grouppatch2}\widehat{\mathcal{G}}_{i,j}=\left\{\widehat{\Yq}_{i_1,j_1},\cdots, \widehat{\Yq}_{i_d,j_d}\right\}\!.\end{equation}
By the averaging technique, a good approximation, denoted by $\widehat{\Lq}$,  to the original color image  is reconstructed  from $m$ reconstructed subgroups   $\widehat{\mathcal{G}}_{i,j}$'s.

Now we show that each  $\Xq_{i,j}$ generated by \eqref{e:Xk}  is of numerically low rank.
\begin{lemma}\label{t:deltarank}
Suppose that the  singular value decomposition  of  $\Xq=[\xq_1,\ldots,\xq_{n_2}]\in\Qmn$ with $n_1\ge n_2$ is
$$\Xq=\Uq\Sigma\Vq^H,$$
where $\Sigma=\texttt{diag}(\sigma_1,\sigma_2,\cdots,\sigma_{n_2}),~\sigma_{1}\ge\sigma_{2}\ge\cdots\sigma_{n_2}\ge 0.$
Denote $\Vq=[\vq_1,\cdots,\vq_{n_2}]$ and  let $r~(\le {n_2})$ be the least positive integer such that
\begin{equation}\label{e:r}
\sum\limits_{j=1}^{r-1}(\sigma_j^2-\sigma_r^2)|\wq_j|^2\ge\sum\limits_{j=r+1}^{n_2}(\sigma_r^2-\sigma_j^2)|\wq_j|^2
\end{equation}
holds for any two columns $\vq_i$ and $\vq_j$,  where $[\wq_1,\ldots,\wq_n]^T=\vq_i-\vq_j$.
For any given $\delta>0$, if  any two columns $\xq_i$ and $\xq_j$ satisfy
$$\|\xq_i-\xq_j\|_2\le \sqrt{2} \delta, $$
then $\sigma_{r}\le \delta$ and thus the $\delta$-rank of $\Xq$ is less than  $r$.
\end{lemma}
\begin{proof} Let $\vq_i$ and $\vq_j$ be two different columns of unitary matrix $\Vq$. Then $\vq_i^H\vq_i=\vq_j^H\vq_j=1$ and $\vq_i^H\vq_j=\vq_j^H\vq_i=0$. Thus,
$\|\vq_i-\vq_j\|_2^2=(\vq_i-\vq_j)^H(\vq_i-\vq_j)=\vq_i^H\vq_i-\vq_i^H\vq_j-\vq_j^H\vq_i+\vq_i^H\vq_i=2$.
The proof directly follows the facts that
$\|\vq_i-\vq_j\|_2=\sqrt{2}$,
$$\|\xq_i-\xq_j\|_2=\|\Uq(\xq_i-\xq_j)\|_2=\|\Sigma(\vq_i-\vq_j)\|_2,$$
and that  if $r$ satisfies \eqref{e:r} then
$\|\sigma_r(\vq_i-\vq_j)\|_2\le\|\Sigma(\vq_i-\vq_j)\|_2.$
\end{proof}

Suppose the distance function for each patch group $\mathcal{G}_{i,j}$ defined as in \eqref{e:grouppatch} 
  is uniformly defined as $f(\Yq_{s,t},\Yq_{i,j})=\|\Yq_{s,t}-\Yq_{i,j}\|_F=\|{\tt vec}(\Yq_{s,t})-{\tt vec}(\Yq_{i,j})\|_2$. Lemma \ref{t:deltarank} can be used to analysis the numerical low-rank property of each key patching matrix $\Xq_{i,j}$.
Note that $\Xq_{i,j}$  can be seen as a sum of a low-rank matrix and a perturbation,
\begin{equation}\label{e:Xk2}
\Xq_{i,j}=P_{\Omega_{i,j}}(\Lq_{i,j}+\Sq_{i,j}),
\end{equation}
where $\Lq_{i,j}$ and $\Sq_{i,j}$ are the patch matrices of original image $\Lq_0$ and sparse noise matrix $\Sq_0$, respectively, and $\Omega_{i,j}$ is the patched region of available data.

\begin{theorem}\label{cor:deltarank}
Given $\delta>0$,  suppose each  subgroup $\mathcal{G}_{i,j}$ is generated by \eqref{e:grouppatch} under the metric
$$\|\Yq_{s,t}-\Yq_{i,j}\|_F\le \frac{\sqrt{2}}{2}\delta, $$
 and the associated patching matrix $\Xq_{ij}\in \mathbb{Q}^{  (wh) \times d }$ 
 (with $wh\ge d$)
has  the singular value decomposition,
\begin{equation}\label{e:svdXij}\Xq_{i,j}=\Uq_{i,j}\Sigma_{i,j}\Vq_{i,j}^H,
\end{equation}
where
$\Sigma_{ij}=\texttt{diag}(\sigma_1,\sigma_2,\cdots,\sigma_d),~\sigma_1\ge\sigma_2 \cdots\ge\sigma_{d}\ge 0.$
Let $r_{i,j}~(\le {d})$ be the least positive integer such that
\begin{equation}\label{e:rij}
\sum\limits_{k=1}^{{r_{i,j}}-1}(\sigma_{k}^2-\sigma_{r_{i,j}}^2)|\wq_{k}|^2\ge \sum\limits_{k=r_{i,j}+1}^d (\sigma_{r_{i,j}}^2-\sigma_k^2)|\wq_k|^2,
\end{equation}
where $[\wq_1,\ldots,\wq_d]^T=\vq_i-\vq_j$,
holds for any two columns of  $\Vq_{i,j}=[\vq_1,\cdots,\vq_{d}]$.
Then $\sigma_{r_{i,j}}\le \delta$ and thus the $\delta$-rank of $\Xq_{i,j}$ is less than  $r_{i,j}$.
\end{theorem}
\begin{proof} By Lemma \ref{t:deltarank}, we only need to show that the distance between any two columns of $\Xq_{i,j}$ is less than or equal to $\sqrt{2}$.
Since each element of $\mathcal{G}_{i,j}$ satisfy $\|\Yq_{s,t}-\Yq_{i,j}\|_F\le \frac{\sqrt{2}}{2}\delta$,
 $\|\Yq_{s_1,t_1}-\Yq_{s_2,t_2}\|_F
 =\|\Yq_{s_1,t_1}-\Yq_{i,j}+\Yq_{i,j}-\Yq_{s_2,t_2}\|_F\le \|\Yq_{s_1,t_1}-\Yq_{i,j}\|_F+\|\Yq_{i,j}-\Yq_{s_2,t_2}\|_F\le \sqrt{2}\delta$
 holds for any two columns ${\tt vec}(\Yq_{s_1,t_1})$ and ${\tt vec}(\Yq_{s_2,t_2})$ of $\Xq_{i,j}$.
\end{proof}

\subsection{Practical QMC algorithm}
\label{ss:admm}
Now we are at the position 
to solve the resulting QMC problem corresponding to each $\Xq_{i,j}$ defined by \eqref{e:Xk}:
\begin{equation}\label{e:qmcpatchmodel}
\begin{aligned}
\min_{\Lq,\Sq}&~\|\Lq\|_{\ast}+\lambda\|\Sq\|_1\\
\rm{s.t.}&~\mathcal{P}_{\Omega_{i,j}}(\Lq+\Sq)=\Xq_{i,j}.
  \end{aligned}
\end{equation}
 For the convenience of presentation, we omit the subscripts of $\Xq_{i,j}$ and $\Omega_{i,j}$ in \eqref{e:qmcpatchmodel}, and use $\Xq$ and $\Omega$ instead.
 
To solve \eqref{e:qmcpatchmodel}, we present a practical QMC algorithm under the ADMM framework with theoretical convergence guarantee \cite{boyd2011distributed}.
 Firstly, we reformulate \eqref{e:qmcpatchmodel} as the following equivalent optimization problem by introducing two auxiliary variables  $\Pq$ and ${\bf Q}$,
\begin{equation} \label{cmodel}
\begin{array}{rl}
\underset{{\bf L}, {\bf S}, {\bf P}, {\bf Q}}{\min}&
\left\Vert {\bf L} \right \Vert_{\ast} + \lambda \left\Vert {\bf S}\right\Vert _{1}\\
\ {\rm s.t.} &
\mathcal{P}_{\Omega}\left({\bf P}+{\bf Q}\right)={\bf X}, {\bf L}={\bf P}, {\bf S}={\bf Q}.
\end{array}
\end{equation}
 The augmented
Lagrangian function is given as follows by attaching multipliers ${\bf Y}$ and ${\bf Z}$:
\begin{equation} \label{ucmodel}
\begin{array}{rl}
\underset{{\bf L}, {\bf S}, {\bf P}, {\bf Q}}{\min}&
\left\Vert {\bf L} \right \Vert_{\ast} + \lambda \left\Vert {\bf S}\right\Vert _{1}+\frac{\mu}{2} \left\Vert {\bf L-\bf P+\bf Y/\mu}\right\Vert _{F}^2\\
&+\frac{\mu}{2} \left\Vert {\bf S-\bf Q+\bf Z/\mu}\right\Vert _{F}^2\\
\ {\rm s.t.} &
\mathcal{P}_{\Omega}\left({\bf P}+{\bf Q}\right)={\bf X},
\end{array}
\end{equation}
where  $\mu$ is the penalty parameter for linear constraints to be satisfied. Such joint minimization problem can be decomposed into two easier
and smaller subproblems such that two groups of variables $[\bf{L},\bf{Q}]$ and $[\bf{S},\bf{P}]$  can
be minimized in an alternating order, followed by the update of multipliers  ${\bf Y}$ and ${\bf Z}$. 
This strictly follows the ADMM framework in  \cite{boyd2011distributed}, and thus,  the theoretical convergence is guaranteed.  Moreover,   $\bf{L}$ and $\bf{Q}$ are decoupled in the $[\bf{L},\bf{Q}]$ subproblem, and they  can be  solved separately. It is also true for   $\bf{S}$ and $\bf{P}$ in the $[\bf{S},\bf{P}]$ subproblem.  The detailed analysis is presented in the supplementary material.

We summarize  above steps of solving \eqref{e:qmcpatchmodel}  into Algorithm 1,   in which a quaternion matrix  $\Xq=X_0+X_1\iq+X_2\jq+X_3\kq$ is stored as $X=[X_0~X_1~X_2~X_3]$.  The computational cost of Algorithm 1 at per iteration is $O(n_1 n_2+\min (n_1n_2^2,n_1^2 n_2))$ for each  observed quaternion matrix $\Xq\in \Qmn$. 
\begin{algorithm}\label{a:qmc_admm}{\bf Algorithm 1 (QMC Algorithm).}
Given an observed quaternion matrix $\Xq=X_0+X_1\iq+X_2\jq+X_3\kq\in\Qmn$, this algorithm computes a low-rank quaternion matrix    $\Lq=L_0+L_1\iq+L_2\jq+L_3\kq$  and a sparse quaternion matrix $\Sq=S_0+S_1\iq+S_2\jq+S_3\kq$  satisfying  \eqref{e:qmcpatchmodel}.
\begin{enumerate}
\item[$1.$] Input: $X=[X_0~X_1~X_2~X_3]$;  the set of the indexes of known pixels $\Omega$; the set of the indexes of unknown pixels $\overline{\Omega}$; $L = X$; $S=P=Q=Y=Z=\texttt{zeros}(n_1, 4*n_2)$; $\mu,\lambda>0$.
\item[$2.$] While not converge
\item[$3.$] \quad  Update $L$ and  $Q$:
 \item[]\qquad   $L = \texttt{approxQ}(P -  (1/\mu)*Y, 1/\mu)$;
 \item[]\qquad   $Q(\Omega)=X(\Omega)-P(\Omega)$; \quad
             $Q(\overline{\Omega})=S(\overline{\Omega})+Z(\overline{\Omega})/\mu$;
\item[$4.$] \quad  Update $S$ and $P$:
\item[] \qquad $S =\texttt{shinkQ}(Q-  (1/\mu)*Z, \lambda/\mu)$;
\item []       \qquad  $P(\Omega)=(\mu L(\Omega)+\mu X(\Omega)-\mu S(\Omega)+Y(\Omega)-$
\item[]         \qquad $Z(\Omega))/2/\mu$;  \quad $P(\overline{\Omega})=L(\overline{\Omega})+Y(\overline{\Omega})/\mu$;
\item[$5.$]  \quad  Update $Y$ and $Z$:
\item[]         \qquad  $Y=Y+\mu*(L-P)$;   $Z = Z + \mu*(S-Q)$;
\item[$6.$] end
\end{enumerate}
\end{algorithm}

\subsection{NSS-based QMC Algorithm}
Based on above analysis, we propose the NSS-based QMC algorithm, whose  pseudo code is proposed in  Algorithm 2.

\begin{algorithm}\label{a:pqmc}{\bf Algorithm 2 (NSS-based QMC  Algorithm).}
Given an observed color image $\Xq=X_0+X_1\iq+X_2\jq+X_3\kq\in\Qmn$ with $X_0\equiv 0$, this algorithm computes a   reconstruction    $\widehat{\Lq}=\widehat{L}_0+\widehat{L}_1i+\widehat{L}_2j+\widehat{L}_3k$ with $\widehat{L}_0\equiv 0$ and $\widehat{L}_1,~\widehat{L}_2,~\widehat{L}_3$ denoting the red, green, blue color channels, respectively.
\begin{enumerate}
\item[$1.$]  Input:  The color image  $X=[X_0~X_1~X_2~X_3]$;  the set of the indexes of known pixels $\Omega$;  the patch size  $w\times h$ with $0<w\le n_1$ and $0<h\le n_2$; the number of key patches $m$; a tolerance  $\delta>0$;  the dimension of the patching group $d$.
\item[$2.$] While not converge
\item[$3.$]  \qquad Generate the group of  patches $\mathcal{G}$ as in \eqref{e:Gpatches}. Choose $m$  key patches  $\Yq_{i,j}$ with $(i,j)\in\Omega$.
\item[$4.$]   \qquad  Search   $d$  similar patches to  $\Yq_{i,j}$ and gather them into $\mathcal{G}_{i,j}$. Generate  a quaternion matrix $\Xq_{i,j}$ as in \eqref{e:Xk}.
\item[$5.$]  \qquad   Apply the QMC algorithm  (Algorithm 1) to compute  a  low-rank approximation $\widehat{\Lq}_{i,j}$ of $\Xq_{i,j}$.   Reconstruct the  subgroups of  approximation patches $\widehat{\mathcal{G}}_{i,j}$.
\item[$6.$] \qquad   Compute the reconstruction $\widehat{\Lq}$  from $m$ reconstructed subgroups   $\widehat{\mathcal{G}}_{i,j}$.
\item[$7.$]  end

\end{enumerate}
\end{algorithm}

In Algorithm 2,  the overlapped patches are averaged with weights generated by the similarity  to key patch; see \eqref{e:patchsimilar} for instance.
It is satisfied that  $\|{\tt vec}(\Yq_{i_s,j_s})-{\tt vec}(\Yq_{i,j})\|_2=\|\Yq_{i_s,j_s}-\Yq_{i,j}\|_F\le \frac{\sqrt{2}}{2}\delta$.  (By Lemma \ref{t:deltarank},     the $\delta$-rank of $\Xq_{i,j}$ is less than  $r$ which is the index of the first singular value such that $\sigma_r(\Xq_{i,j})\le \delta$.)
 From Theorem \ref{cor:deltarank}, it is possible to get  low $\delta$-rank quaternion matrices,  $\Xq_{i,j}$'s, in the NSS-based QMC by setting small $\delta$.  This is indicated in Fig.\ref{Fig:patchedidea}.
From Theorem 2 in \cite{jia2018quaternion}, the lower  the rank  of the matrix is, the accurate the matrix is recovered by QMC.  The idea of NSS-based QMC provides a method to construct low-rank matrices, on which the QMC is applied.
\begin{remark}
Finding the  patches similar to a given  patch by block matching,  NSS-based QMC and BM3D  group them into  a quaternion matrix and a third order tensor, respectively. The
 differences are in  two folds: (1) NSS-based QMC is for robust completion and BM3D is
 for denoising; (2) the redundancy of similar patches is characterized by
 the low-rankness of the quaternion matrix in NSS-based QMC and the redundancy of similar patches is characterized  by  the sparsity of the transformed third-order tensor in BM3D.

So we will not compare NSS-based QMC with  BM3D \cite{dabov2007image} in our numerical experiments since it is not appropriate for solving color image inpainting problem.
In fact, if BM3D without modification is  applied into Example \ref{e:ex1}  then the obtained reconstructions will contain unknown pixels and some noise.
 \end{remark}

\section{Large-scale Color Image and  Video Inpainting}\label{sec:large}
The NSS-based QMC method can not be directly applied to handle the large-scale color video inpainting problem.  In this section, a new tensor NSS-based QMC (TNSS-based QMC)  method is presented for  color video inpainting based on the quaternion tensor representation.

 At first, we present a new TNSS-based QMC algorithm for color video inpainting.  Under the quaternion representation, a color video is represented by a third-order quaternion tensor, whose slices are quaternion matrices. When the scale of each slice of color video is huge, it is practical to set a  window for searching similar patches. 
For a given size $D>0$,  the searching window corresponding to the key patch $\Yq_{i,j}$ with the centre $(i,j)$ is defined by
\begin{equation}\label{e:ND}
\mathcal{N}_D(i,j) =\left\{ (s,t): \|(s,t)-(i,j)\|_{\infty} \leq \frac{D-1}{2} \right\},
\end{equation}
where  $\| \cdot \|_{\infty}  $ denotes the sup norm:
$\| (s,t) - (i,j)\|_{\infty} = \max \{ |s-i|,|t-j| \} $.
In other words, $\mathcal{N}_D(i,j)$ is the centre index set of patches 
 with
the centre $(s,t)$  near to  $(i,j)$. 

Let  $\mathcal{X}\in\mathbb{Q}^{n_1\times n_2 \times n_3}$  be a third-order  quaternion tensor and let its $k$-th frontal slice be  $\Xq(k):=\mathcal{X}(:,:,k)\in\mathbb{Q}^{n_1\times n_2}$, $k=1,\cdots,n_3$.
  Once the key patch $\Yq_{i,j}(k)$ and the searching window $\mathcal{N}_D(i,j)$  on the $k$-th slice $\mathcal{X}(:,:,k)$ is set, we use the same index $(i,j)$ of the key patch  and the size of searching window to other slices, and search the similar patches  on all slices.
 More specifically,
let the patches in the fixed searching window on the $k$-th slice be grouped in
 \begin{equation}\label{e:gDtensor}
\mathcal{G}^D(k)= \left\{ \Yq_{s,t}(k): (s,t)\in \mathcal{N}_D(i,j) \right\}\!.
\end{equation}
We search $d$ patches that are similar to the key patch $\Yq_{i,j}(k)$ from $\mathcal{G}^D(1),\cdots, \mathcal{G}^D(n_3)$ and gather them into
\begin{equation}\label{e:grouppatch_tensor}
\mathcal{G}_{i,j}^D(k)= \left\{ \Yq_{i_1,j_1}(k_1),\cdots,\Yq_{i_d,j_d}(k_d) \right\}\!,
\end{equation}
where $(i_s,j_s)\in \mathcal{N}_D(i,j)$ and $1\le k_s\le n_3$, $s=1,\cdots,d$.
Then  a quaternion matrix $\Xq^D_{i,j}(k)\in \mathbb{Q}^{  (wh) \times d }$  is generated by
\begin{equation}\label{e:XkDtensor}
\Xq_{i,j}^D(k)=[{\tt vec}(\Yq_{i_1,j_1}(k_1)),\cdots,{\tt vec}(\Yq_{i_d,j_d}(k_d))].
\end{equation}
Then the rest steps are similar to the NSS-based QMC algorithm.  The indication is shown in Fig.\ref{Fig:tensor} and the pseudo code is proposed in Algorithm $3$. Here, the superscript `$~\widehat{}~$' has the similar meaning as in Algorithm $2$.

\begin{figure}[!h]
\begin{center}
\includegraphics[width=0.85\linewidth,height=0.55\linewidth]{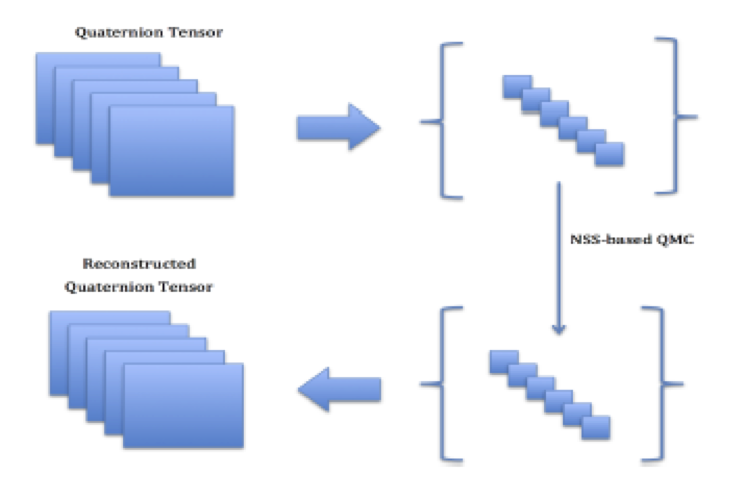}
\end{center}
\caption{{\protect\small {The operating principle of the TNSS-based QMC on quaternion tensor reconstruction.  }}}
\label{Fig:tensor}
\end{figure}

\begin{algorithm}\label{a:pqmc_TNSS}{\bf Algorithm 3 (TNSS-based QMC  Algorithm).}
Given an observed color video $\mathcal{X}\in\mathbb{Q}^{n_1\times n_2 \times n_3}$, this algorithm computes a reconstruction  $\widehat{\mathcal{L}}\in\mathbb{Q}^{n_1\times n_2 \times n_3}$.
\begin{enumerate}
\item[$1.$]  Input:  The color video  $\mathcal{X}\in\mathbb{Q}^{n_1\times n_2 \times n_3}$;  the set of the indexes of known pixels $\Omega$ on the $k$-th frontal slice; the patch size  $w\times h$ with $0<w\le n_1$ and $0<h\le n_2$; the number of key patches $m$; a tolerance  $\delta>0$;  the dimension of the patching group $d$; the size of searching window $D>0$.
\item[$2.$] While not converge
\item[$3.$]  \qquad Choose the  key patch  $\Yq_{i,j}(k)$ and set the searching window $\mathcal{N}_D(i,j)$ on the $k$-th slice.
 Generate a group of  patches $\mathcal{G}^D(k)$  as  in \eqref{e:gDtensor}.

\item[$4.$]   \qquad  Search   $d$  similar patches as in \eqref{e:grouppatch_tensor}. Generate  a quaternion matrix $\Xq_{i,j}^D(k)$ as in \eqref{e:XkDtensor}.
\item[$5.$]  \qquad   Apply the QMC algorithm  (Algorithm 1) to compute  a  low-rank approximation $\widehat{\Lq}_{i,j}^D(k)$ of $\Xq_{i,j}^D(k)$.   Reconstruct the  subgroups of  approximation patches $\widehat{\mathcal{G}}_{i,j}^D(k)$.
\item[$6.$] \qquad   Compute the reconstruction $\widehat{\mathcal{L}}$  from $m$ reconstructed subgroups   $\widehat{\mathcal{G}}_{i,j}^D(k)$.
\item[$7.$]  end

\end{enumerate}
\end{algorithm}

TNSS-based QMC is in fact a novel quaternion tensor completion method. The corresponding exact recovery theory can be similarly  built to the QMC theory in \cite{jia2018quaternion}.  This method is feasible to solve the color video inpainting problems with incomplete and corrupted pixels.
Especially, it excels at dealing with one of the challenging  problems,
 i.e., all frames of  the color video miss pixels at the same positions.
That is, TNSS-based QMC can efficiently reconstruct quaternion tensors  from incomplete and corrupted tubes.
  Although difficult to prove theoretically, this capability of TNSS-based QMC obviously  relies on two facts: the quaternion representation (each color pixel is represented by a purely quaternion)  can fully preserve color information \cite{jia2018quaternion} and the NSS technique leads to low-rank prior (see Theorem \ref{cor:deltarank}).    We will perform numerical verification in Example V.4.

The  strategy of searching similar patches in Algorithm 3 can also be introduced into Algorithm 2 to tackle the large-scale color image inpainting problem.
  For a given key patch $Y_{i,j}$ and a searching window $\mathcal{N}_D(i,j)$ defined in \eqref{e:ND},
 we find $d$  patches  similar to the key patch  from
\begin{equation*}
\mathcal{G}^D= \left\{ \Yq_{s,t}: (s,t)\in \mathcal{N}_D(i,j) \right\},
\end{equation*}
 and denote them by
$$\mathcal{G}_{i,j}^D= \left\{ \Yq_{i_1,j_1},\cdots,\Yq_{i_d,j_d}\right\}\!,$$
where $(i_s,j_s)\in \mathcal{N}_D(i,j)$, $s=1,\cdots, d$.
 Then we  construct  a quaternion matrix 
\begin{equation}\label{e:XkD}
\Xq_{i,j}^D=[{\tt vec}(\Yq_{i_1,j_1}),\cdots,{\tt vec}(\Yq_{i_d,j_d})]\in \mathbb{Q}^{  (wh) \times d }.
\end{equation}
Replacing  $\Xq_{i,j}$ in Step 4 of Algorithm $2$ with $\Xq_{i,j}^D$, we get the NSS-based QMC algorithm with searching window, which can handle the large-scale color image inpainting problem.

To enhance the visual comfort of the reconstructed color image, we usually  select  searching windows with a proper size (decided by experiments), and the patches are  often reconstructed more than once.
 The similarity of two patches is estimated  by weighted Euclidean distance as
follow
\begin{eqnarray}
\nonumber \|\Yq_{s,t} - \Yq_{i,j}\|_{a}^{2} =\qquad\qquad\qquad\qquad\qquad\\
\label{e:patchsimilar}
\frac{\sum\limits_{(\varsigma,\vartheta)\in \mathcal{N}_D(0,0)} a(\varsigma,\vartheta)\|\Yq_{s+\varsigma, t+\vartheta}
 - \Yq_{i+\varsigma, j+\vartheta}\|_2^{2} }{\sum\limits_{(\varsigma,\vartheta)\in \mathcal{N}_D(0,0)} a(\varsigma,\vartheta)},
\end{eqnarray}
with $a(\varsigma,\vartheta)>0$ being some fixed weights   (see, e.g., \cite{buades2005review,buades2005non}).
This similarity suggests a weight to the reconstruction of patch $\Yq_{s,t}$ at the step of applying the searching window   $\mathcal{N}_D(i,j)$.

\begin{remark}
TNSS-based QMC is different to recently proposed methods:  LRQA \cite{chen2019} and LRQTC \cite{miao2020}.  LRQA \cite{chen2019}  is proposed based on nonconvex functions and is applied to each frame of color video.
In LRQTC \cite{miao2020},  the global low-rank prior to quaternion tensor is encoded as the nuclear norm of  unfolding quaternion matrices of large-scale. In TNSS-based QMC, the low-rank prior is of small-scale quaternion matrices generated by similar neighbourhood patches from  $\mathcal{G}_{i,j}^D(k)$ and the local prior information in searching window are  from  all frames.
\end{remark}

\begin{remark}
TNSS-based QMC is also different to  the  tensor methods proposed in \cite{jing17,liu2019,Johnttmac,zhang2021,CY2016,zhao2022,yang2020}, which have been applied to the restoration of color images and videos.
In the tensor methods,  a color image is represented by a third-order tensor and  a color video is represented by a fourth-order tensor, and their entries are real numbers which denote the values of red, green or blue channels.  The operation process follows the real number operation rules and  the relationship of  the values of three color channels of the same color pixel is often ignored.  They apply the low-rank prior of real algebraic  structures.
Differently,  TNSS-based QMC  processes the quaternion operation and always operates the values of the three color channels of the same color pixel as a whole by encoding them into a quaternion \cite{jia2018quaternion}.  In numerical experiments,  we will make a comparison with several tensor methods for which the corresponding codes can be found.
\end{remark}

\section{Numerical Examples}\label{sec:ex}
In this section, we apply the newly proposed NSS-based QMC and TNSS-based QMC approaches  to  solving the inpainting problems of color images and videos.  We compare our approaches with the state-of-the-art methods in the literature.  All these experiments were performed in MATLAB on a personal computer with 2.4 GHz Intel Core i7  processor and 8 GB 1600MHz DDR3 memory.

\begin{example}[{\bf Color Image Inpainting}]
\label{e:ex1}
In this experiment, we apply the proposed  NSS-based QMC approach to color image inpainting  from incomplete and corrupted entries,   and   compare it with  four well-known methods:
\begin{itemize}
\item  QMC--Quaternion matrix completion  algorithm  \cite{jia2018quaternion}.
\item  TC--Tensor completion  algorithm  \cite{jing17}.
\item  TMac$\_{}$TT--A multilinear matrix factorization model to approximate the tensor train rank of a tensor \cite{Johnttmac}.
\item  LRQA-L--Low-rank quaternion approximation with laplace function \cite{chen2019}.
\item  LRQMC--Low-rank quaternion matrix completion algorithm \cite{jk2022}.
\end{itemize}

\begin{table}
\caption{PSNR and SSIM values of reconstructions by QMC, TC,
TMac$\_{}$TT, LRQA-L, LRQMC, and NSS-based QMC.} \label{tab:qmcVSpqmc}
\begin{threeparttable}
\begin{center}
\renewcommand{\arraystretch}{1.1} \vskip-3mm {\fontsize{8pt}{\baselineskip}%
\selectfont \scalebox{0.90}{

\begin{tabular}{c|c|cc|cc|cc|cc}
\hline
Images    & \multirow{2}{*}{ \diagbox{Methods}{$(1-\rho,\gamma)$}}  &\multicolumn{2}{c|}{ $(10\%,0\%)$ } &\multicolumn{2}{c|}{ $(50\%,0\%)$ }& \multicolumn{2}{c}{ $(50\%,10\%)$ }& \multicolumn{2}{c}{ $(0\%,10\%)$ } \\
   (Sizes)         &  &PSNR&SSIM &PSNR&SSIM&PSNR&SSIM&PSNR&SSIM \\ \hline\hline
  \multirow{5}{*}{
  $ \begin{array}{c}{\rm F16}\\
   (512\times 512)\\
   \end{array}$
   }
&QMC
            &   29.63  & 0.9580  &   27.91  &  0.9199 &  26.92 &   \underline{0.8861} & 29.03&0.9575 \\

&TC
            &   31.51  &  0.9667  &   29.14  &  0.9255 &  \underline{27.74} &   0.8786 &\underline{30.93} &\underline{0.9740}\\
&TMac$\_{}$TT
            &   \underline{38.88} & \underline{0.9968}  &   \underline{32.10}  &  {\bf 0.9798}  &  24.22 &   0.7973 &23.41 &0.7894\\
&LRQA-L\tnote{1}
            &   32.58  &  0.9713  &   31.78  &  \underline{0.9682} &  18.08 &   0.4928 &25.94 &0.7840\\
&LRQMC
            &   {\bf 42.11}  &  {\bf 0.9972}  &   31.54  &  0.9583 &  24.21 &   0.5030 &24.82 &0.6044\\
&NSS-based QMC
            & 34.25    &  0.9756    &  {\bf 32.59}  &  0.9640  & {\bf 31.64}  &  {\bf 0.9542} & {\bf 33.86} &{\bf 0.9819}\\
\hline
 \multirow{5}{*}{
  $ \begin{array}{c}{\rm Pepper}\\
   (512\times 512)\\
   \end{array}$
   }
&QMC
            &  30.05  &  0.9827  & 28.27   & \underline{0.9721} &  \underline{27.29}   & \underline{0.9637}&\underline{29.75} &\underline{0.9555} \\
&TC
            &  29.60  &  0.9794  & 27.99   & 0.9691 &  27.13   & 0.9615 &29.49 &0.9539\\
&TMac$\_{}$TT
            &   \underline{37.79}  &   \underline{0.9936}  &   31.19  &  0.9646 &  20.54 &   0.6337 &19.30 &0.6043\\
&LRQA-L
            &   25.54  &  0.9276  &   \underline{32.99}  &  0.9684 &  17.96 &   0.4968 & 26.09&0.6900 \\
&LRQMC
            &   {\bf 39.25}  &  {\bf 0.9940}  &   31.54  &  0.9583 &  24.21 &   0.5030 &24.82 &0.6044\\
&NSS-based QMC
            & 34.72  &   0.9927& {\bf 33.22}  &  {\bf 0.9898}&  {\bf 32.50}   & {\bf 0.9875} &{\bf 34.47}& {\bf 0.9794}\\
\hline
 \multirow{5}{*}{
  $ \begin{array}{c}{\rm Flowers}\tnote{2}\\
   (500\times 362)\\
   \end{array}$
   }
&QMC
            &   25.42&    0.9283  &    24.00&    0.8895 & 23.45&    \underline{0.8571} & 25.52&0.8668\\
&TC
            &   25.82&    0.9329  &    24.19&    0.8876 & \underline{23.66}&    0.8539 & \underline{26.13}&\underline{0.8692}\\
&LRQA-L
            &   26.58  &  0.9316  &   \underline{27.55}  &  \underline{0.8780} &  17.92 &   0.4259 & 23.04&0.6290 \\
&LRQMC
            &   {\bf 36.91}  &  {\bf 0.9754}  &   27.43  &  0.8288 &  23.33 &   0.2973 &24.45 &0.4689\\
&NSS-based QMC
            &  \underline{28.83}&    \underline{0.9517}&  {\bf 27.55}&   {\bf 0.9369}&  {\bf  27.12}    &{\bf 0.9283}& {\bf 28.78}&{\bf 0.9123}\\
\hline
 \multirow{5}{*}{
  $ \begin{array}{c}{\rm Lena}\\
   (512\times 512)\\
   \end{array}$
   }
&QMC
            &  29.73  &  0.9812  & 28.17  &  0.9725 & 27.30 &  0.9656 &29.52 &0.9489\\
&TC
            &  31.23  &  0.9850  & 29.22  &  \underline{0.9760} & \underline{28.17}  &  \underline{0.9689}& \underline{31.22}&\underline{0.9606} \\
&TMac$\_{}$TT
            &    \underline{38.57} &  0.9824  &   \underline{31.02}  &  0.9155 &  21.10 &   0.6723 &19.85 &0.6469\\
&LRQA-L
            &   29.22  &  0.9285  &   30.94  &  0.8808 &  18.06&   0.5062 & 26.47&0.7290 \\
&LRQMC
            &   {\bf 40.76}  &  {\bf 0.9945}  &   32.44  &  0.9594 &  24.97 &   0.5282 &25.78 &0.6469\\
&NSS-based QMC
            & 35.63&  \underline{0.9935}&  {\bf 33.88}  &  {\bf 0.9904}&  {\bf  33.11} &   {\bf 0.9882}&{\bf 35.64} & {\bf 0.9806}\\
\hline
 \multirow{5}{*}{
  $ \begin{array}{c}{\rm House}\\
   (256\times 256)\\
   \end{array}$
   }
&QMC
            &  29.11&     0.9672 & 27.46&   0.9495 & 26.47&    0.9339& 29.06& 0.9235\\
&TC
            & 31.61&    0.9743 & 29.35&    \underline{0.9573} & \underline{28.03}&    \underline{0.9401} &\underline{31.69} &\underline{0.9268}\\
&TMac$\_{}$TT
            &   \underline{39.74} &  0.9820  &   \underline{32.52}  &  0.9130 &  21.70 &   0.4811&21.48 &0.4384 \\
&LRQA-L
            &   35.51  &  0.9381  &   30.39  &  0.8924 &  17.97 &   0.3641 & 23.40&0.7740\\
&LRQMC
            &   {\bf 41.42}  &  \underline{0.9840}  &   32.73  &  0.9053 &  25.35 &   0.2925 &23.13 &0.4384\\
&NSS-based QMC
            &  35.42&    {\bf 0.9887} & {\bf  33.86}&  {\bf  0.9841}& {\bf  32.87}&   {\bf  0.9798}& {\bf 35.35}&{\bf  0.9597}\\
\hline
 \multirow{5}{*}{
  $ \begin{array}{c}{\rm Barbara}\\
   (256\times 256)\\
   \end{array}$
   }
&QMC
            &  29.31&     0.9031 & 27.19&   0.8359 & 26.16&    0.7723& 29.11& 0.8963\\
&TC
            & 32.49&    0.9554 & 28.03&    0.8477 & \underline{26.52}&    \underline{0.7698} &\underline{32.07} &\underline{0.9422}\\
&TMac$\_{}$TT
            &   \underline{36.91} &  \underline{0.9792}  &   \underline{29.80}  &  0.8997 &  20.22 &   0.5115&19.16 &0.4894 \\
&LRQA-L
            &   25.54  &  0.9213  &   29.14  &  \underline{0.9148} &  19.34 &   0.4803 & 25.64&0.7310\\
&LRQMC
            &   {\bf 37.93}  &  {\bf 0.9797}  &   29.51  &  0.8771 &  24.13 &   0.3543 &24.78 &0.4894\\
&NSS-based QMC
            &  32.80&    0.9564 & {\bf  31.17}&  {\bf  0.9376}& {\bf  30.42}&   {\bf  0.9236}& {\bf 32.55}&{\bf  0.9533}\\
\hline
 \multirow{5}{*}{
  $ \begin{array}{c}{\rm Pallon}\\
   (256\times 256)\\
   \end{array}$
   }
&QMC
            &  31.94&     0.9486 & 30.46&   0.176 & \underline{29.89}&    \underline{0.8790}& 31.98& \underline{0.9441}\\
&TC
            & 31.83&    0.9455 & 30.28&    0.9119 & 29.65&    0.8635 &\underline{31.96} &0.9393\\
&TMac$\_{}$TT
            &   \underline{41.69} &  \underline{0.9872}  &   \underline{34.78}  &  \underline{0.9416} &  18.00 &   0.3267&19.16 &0.4894 \\
&LRQA-L
            &   38.55  &  0.9723  &   32.01  &  0.9259 &  17.74 &   0.4765 & 27.24&0.6780\\
&LRQMC
            &   {\bf 42.83}  &  {\bf 0.9893}  &   34.58  &  0.9357 &  24.96 &   0.2368 &24.98&0.3150\\
&NSS-based QMC
            &  37.48&    0.9785 & {\bf  36.03}&  {\bf  0.9691}& {\bf  35.62}&   {\bf  0.9617}& {\bf 37.48}&{\bf  0.9772}\\
\hline
 \multirow{5}{*}{
  $ \begin{array}{c}{\rm Pens}\\
   (256\times 256)\\
   \end{array}$
   }
&QMC
            &  26.95&     0.9233 & 24.78&   0.8384 & 23.73&    0.7539& 27.15& 0.9134\\
&TC
            & 30.18&    0.9565 & 26.65&    0.8551 & \underline{25.09}&    \underline{0.7621} &\underline{30.29} &\underline{0.9429}\\
&TMac$\_{}$TT
            &   \underline{37.18} &  \underline{0.9822}  &   26.76  &  0.8503 &  19.57 &   0.5150&18.59 &0.4987 \\
&LRQA-L
            &   28.79  &  0.9566  &   28.96  &  \underline{0.9223} &  17.74 &   0.4705 & 25.81&0.7140\\
&LRQMC
            &   {\bf 38.33}  &  {\bf 0.9830}  &   \underline{28.99}  &  0.8771&  23.51 &   0.3633 &24.39 &0.4987\\
&NSS-based QMC
            &  31.19&     0.9561 & {\bf  29.95}&  {\bf  0.9361}& {\bf  29.16}&   {\bf  0.9179}& {\bf 31.18}&{\bf  0.9527}\\
\hline
\end{tabular}}
\begin{tablenotes}
\footnotesize 
\item[1] Here, Algorithm 2 and 3 from \cite{chen2019} are utilized for the last and the first three cases, respectively.
      \item[2] For the color image `{\rm Flowers}', the method `TMac$\_{}$TT' is absent due to the constraints of KA augmentation.
\end{tablenotes}
} \vskip-3mm
\end{center}
\end{threeparttable}
\end{table}

\begin{figure}
\centering
\includegraphics[width=0.99\linewidth,height=0.90\linewidth]{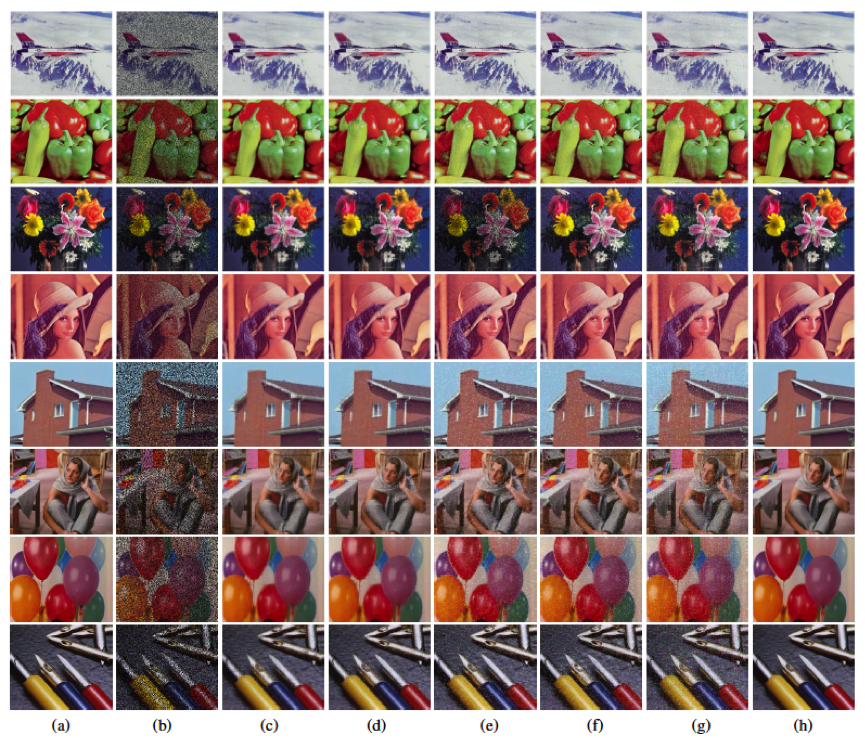}
\caption{Robust color image inpainting: (a) the original color images, (b) the observations ($1-\rho=50\%,\gamma=10\%$),  (c-h) the reconstructions by QMC, TC,  TMac$\_{}$TT, LRQA-L, LRQMC, and  TNSS-based QMC. (The color images can be  viewed better  in zoomed PDF.) }
\label{fig:mainim}
\end{figure}

\begin{figure}
\centering
\includegraphics[width=0.85\linewidth,height=0.40\linewidth]{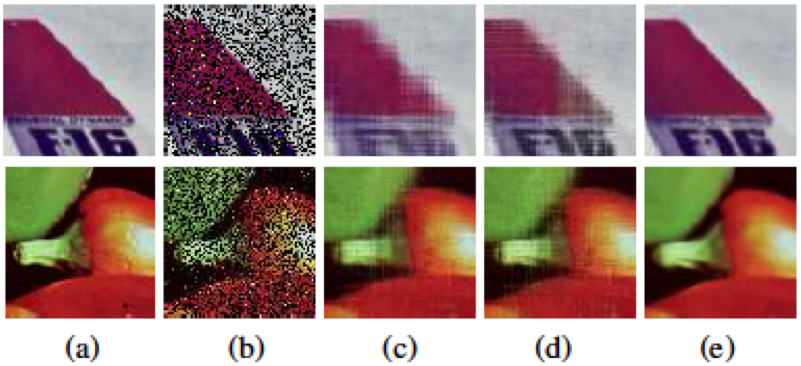}
\caption{The enlarged parts (of size $80\times 80$) of: (a) the  original color images, (b) the observations, (c-e) the reconstructions by QMC, TC, and NSS-based QMC. These are corresponding to that  in the first two rows of  Fig.\ref{fig:mainim}. }
\label{fig:qmcVSpqmc_elarge}
\end{figure}

\begin{figure}[h!]
\centering
\includegraphics[width=0.99\linewidth,height=0.65\linewidth]{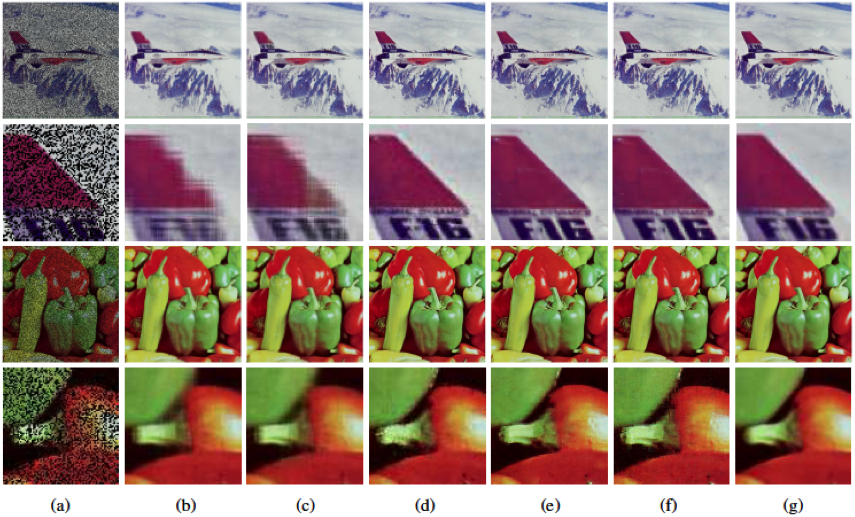}
\caption{Color image inpainting: (a) the observations ($1-\rho=50\%,\gamma=0\%$),  (b-g) the reconstructions  by QMC, TC, TMac$\_{}$TT, LRQA-L, LRQMC, and NSS-based QMC. 
The enlarged parts (of size $80\times 80$) are listed in the 2nd and 4th rows.  }
\label{fig:qmcVSpqmcVSTMacTT}
\end{figure}

Numerical comparisons were implemented on eight  standard color images  and in four  cases with different levels of missingness and noise.
The detailed settings are as follows.
A standard uniform noise is independently and randomly added into $\ell$ pixel locations of red, green and blue channels of color images.
The sparsity of noise components is denoted by $\gamma=\ell/(n_1 n_2)$, where the size of image is $n_1$-by-$n_2$.
Let $\Omega$ be the set of observed entries which are generated randomly and let $\rho=|\Omega|/(n_1 n_2)$ denote the percentage of observed entries.
Four different situations are considered: $(1-\rho, \gamma)=(10\%, 0\%), ~(50\%, 0\%),~(50\%, 10\%),~(0\%, 10\%)$.
The stopping criteria of the alternating direction method of multipliers are that
the norm of the successive iterates is less than $1e$-$4$ and the maximum number of iteration is $500$.
Note that many state-of-the-art real-matrix completion-based methods are not listed in the compared algorithms since QMC was shown in \cite{jia2018quaternion} to have better performance in color image inpainting.

\end{example}

 All numerical results are listed in Table \ref{tab:qmcVSpqmc}. The best PSNR and SSIM values are  shown in bold,  and the second-best ones are underlined. It is worth mentioning that TMac$\_{}$TT  and LRQMC cannot handle the latter two cases with noise ($\gamma>0$) according to \cite{Johnttmac}  and \cite{miao2020}, respectively. LRQA-L  \cite{chen2019} cannot handle the third case with missing pixels and noise  $(1-\rho>0,~\gamma>0)$.
In other words, these three methods are not designed for robust color image inpainting~(i.e., inpainting color image from incomplete and corrupted pixels)
\footnote{   According to the referee reports, our method is compared with TMac$\_{}$TT \cite{Johnttmac}, LRQA-L \cite{chen2019} and LRQMC \cite{miao2020}. All numerical results of these three methods are listed in Table \ref{tab:qmcVSpqmc} for the completeness of comparison. Because the methods are  for completion problem or denosing problem alone, thus the results are not good if the problem is not their targeting problem. }.

Also, TMac$\_{}$TT cannot be directly implemented on the color image `Flower' of size $500\times 362$, since TMac$\_{}$TT is limited to the size of picture because of the ket augmentation (KA) technique, which was originally introduced in \cite{j2005} using an appropriate block addressing.

From Table \ref{tab:qmcVSpqmc}, we can see that  NSS-based QMC performs better than the other methods in most of cases. LRQMC achieves the highest PSNR values in the first case with missing $10\%$ pixels  and without noise (TMac$\_{}$TT achieves  the second highest PSNR values for seven color images). However, when the missing rate rises to $50\%$, NSS-based QMC achieves the highest PSNR values.
For instance,  when only $10\%$ pixels are missing,  the PSNR value of the reconstructed `Pepper' image by  LRQMC is   $39.25$,  which is about  $5$ points higher than  NSS-based QMC ($34.72$);  but when  $50\%$ pixels are missing, the PSNR value of the reconstructed `Pepper' image by NSS-based QMC  is $33.22$, which is about $2$ points higher than  LRQMC ($31.54$). 
The superiority of NSS-based QMC is more obvious when turning to the SSIM values in Table \ref{tab:qmcVSpqmc}.
 Not only that, NSS-based QMC can handle the case with both missing pixels and noise, and performs better than QMC and TC. Overall, the newly proposed algorithm outperforms the other compared  algorithms.

For the visual comparison,  the  original, observed and recovered color images in the case  of  robust color image inpainting with $(1-\rho,\gamma)=(50\% ,10\%)$ are shown in Fig.\ref{fig:mainim}. (Note that since TMac$\_{}$TT cannot directly recover the image `Flowers',  the observation is located at   row $3$ and column $5$.)  It is easily to  find a lot of noise left in the reconstructions of  TMac$\_{}$TT, LRQA-L, and LRQMC  from  Fig.\ref{fig:mainim}(e-g).  This numerically indicates  that these three methods cannot  handle robust color image inpainting, which is not their goal. 
 From Fig.\ref{fig:mainim}(c,d,h),  we can find that QMC, TC and NSS-based QMC  successfully remove noise during  inpainting, since they are designed for robust color image inpainting. 

More importantly, it is observed from Table \ref{tab:qmcVSpqmc} that the PSNR and SSIM values of reconstructions by NSS-based QMC are much higher than those by QMC and TC. 
And from the visual comparison in Fig.\ref{fig:mainim}(c,d,h), the recovered color images by NSS-based QMC are of higher quality.  The enlarged  parts of these reconstructions are listed in Fig.\ref{fig:qmcVSpqmc_elarge}.
The color and edge information is well preserved by  NSS-based QMC.  However, the sharp geometry shapes, likes the edges and angles, are not clearly recovered by QMC and TC.  Moreover, the newly proposed NSS-based QMC overcomes the disadvantage of color mixture of TC.   For instance, the pale yellow appears at the edges of red zones  in Fig.  \ref{fig:qmcVSpqmc_elarge}(row $1$ and column $4$), which is recovered by TC ;  but the different color zones are very clean and  the color edges are very sharp  in  Fig. \ref{fig:qmcVSpqmc_elarge}(row $1$ and column $5$), which is recovered by   NSS-based QMC.  More numerical and visual results can be seen in the supplementary material.

For the completeness of comparison, we also show the reconstructed color images  of  the third and fourth cases $(50\%,0\%)$ and  $(0\%,10\%)$ in  the 1st and 3rd rows of Fig.\ref{fig:qmcVSpqmcVSTMacTT} and Fig.\ref{fig:qmcVSpqmcVSLRQA-L}, and their  enlarged  parts are listed in  the 2nd and 4th rows.
These visual comparisons indicate that  NSS-based QMC can preserve the color information and  color edges well while other methods fail to do this. For example,  LRQMC fails to reach a clean edge of `Pepper'  in  Fig.\ref{fig:qmcVSpqmcVSTMacTT}(row $4$ and column $6$); LRQA-L  fails to recover the letter `F'  in  Fig.\ref{fig:qmcVSpqmcVSLRQA-L} (row $2$ and column $4$).

\begin{figure}
\centering
\includegraphics[width=0.85\linewidth,height=0.65\linewidth]{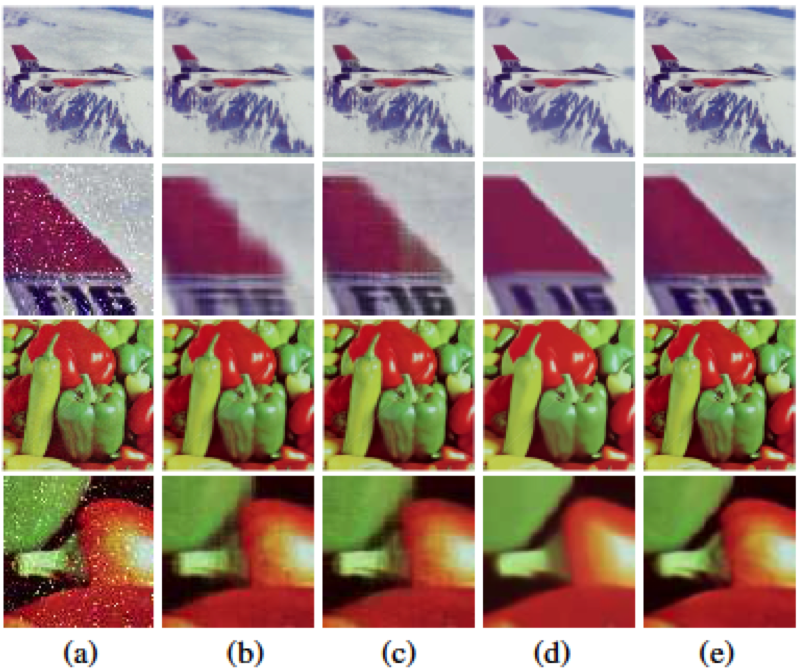}
\caption{Color image denoising: (a) the observations ($1-\rho=0\%,\gamma=10\%$),  (b-e) the reconstructions  by QMC, TC, LRQA-L, and NSS-based QMC. 
The enlarged parts (of size $80\times 80$)  are listed in the 2nd and 4th rows.}
\label{fig:qmcVSpqmcVSLRQA-L}
\end{figure}

 Notice that  TC and TMac$\_{}$TT represent color images by  third-order tensors, while QMC, LRQA-L, LRQMC, and NSS-based QMC by quaternion matrices.
   LRQA-L and NSS-based QMC apply nonlocal self-similarity, but the former does not performs as expected.
 Our methods are not compared with  BM3D \cite{dabov2007image},  LSSC \cite{mairal2009}, NCSR \cite{dong2013}, and WNNM \cite{gu2017weighted} since they are efficient image denoising methods but not proper to solve color image inpainting problem.
For instance,  if BM3D without modification is  applied to solve the color inpainting problem in Example \ref{e:ex1},  then the obtained reconstructions will contain unknown pixels and a few of noise, and their  PSNR and SSIM values  are lower than those in Table  \ref{tab:qmcVSpqmc}.

\begin{figure}[!h]
\centering
\includegraphics[width=0.99\linewidth,height=0.90\linewidth]{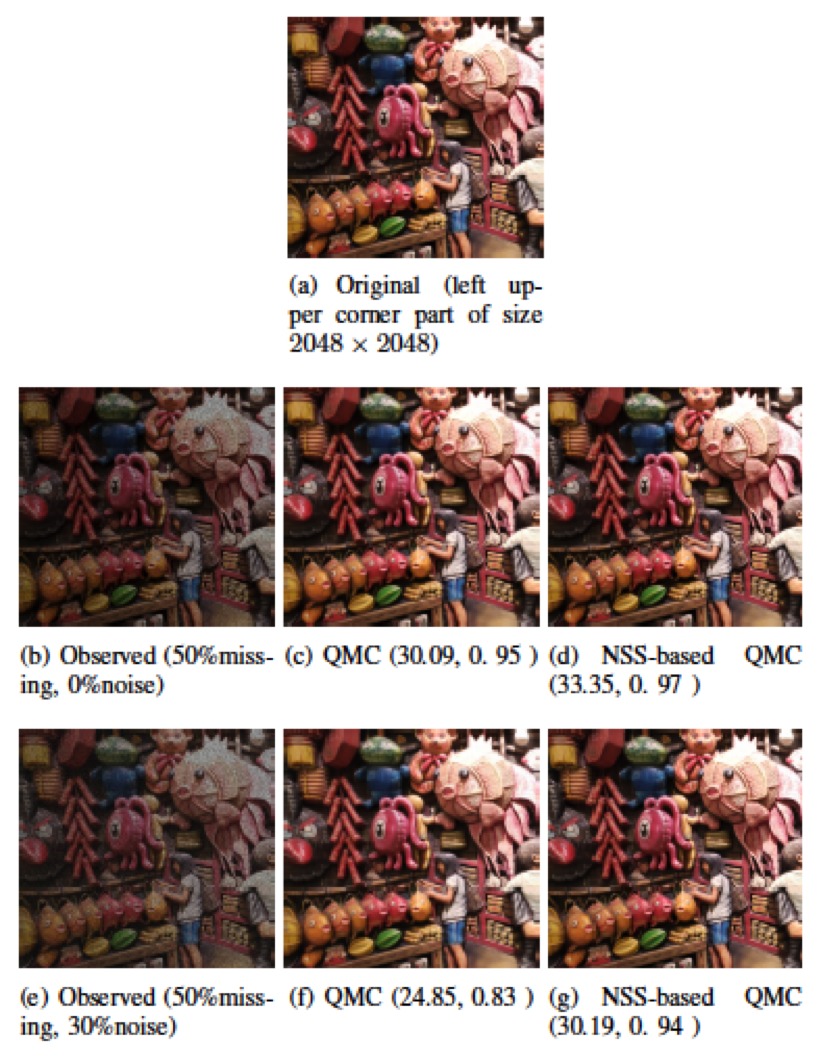}
\caption{Visual comparison of completion results by QMC and NSS-based QMC.}
\label{fig:xyp}
\end{figure}

\begin{figure}[!h]
\centering
\includegraphics[width=0.99\linewidth,height=0.90\linewidth]{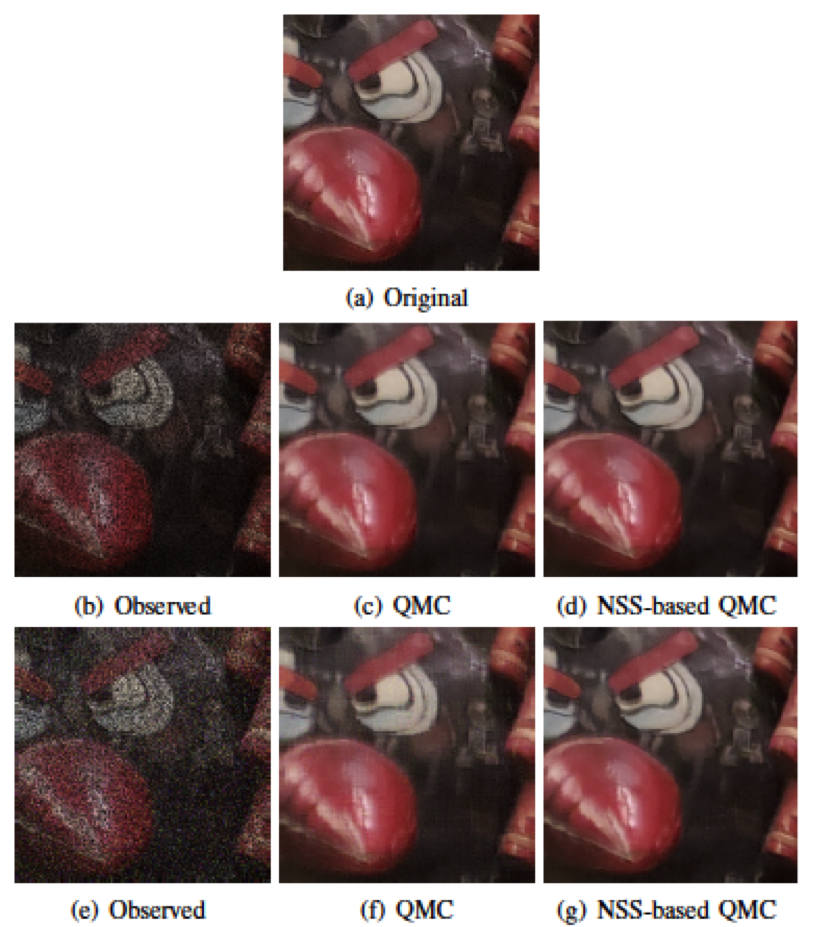}
\vspace{-0.4cm}
\caption{The enlarged parts (of size $170\times 170$) of reconstructions in Fig.\ref{fig:xyp}.}
\label{fig:xyp_elarge}
\end{figure}

\begin{example}[Large-Scale Color Image Inpainting]
\label{e:ex2}
In this experiment, we apply the inpainting problem of a large-scale color image.
The testing image  is a photo of wall painting at the Xi Ying Pan subway station in Hong Kong, taken by the first author in 2019.  The size of original photo is  $3024\times 4032$.
To speed up the process, we partition the original color into 256 blocks of size $189\times 252$ at first and  then apply the proposed NSS-based QMC on each block in parallel. The setting is that the size searching window $D=20$ and  the number of patches in each patching group $d=60$.
The stopping criteria of the alternating direction method of multipliers are that the norm of the successive iterates is less than $1e$-$4$ and the maximum number of iteration is $100$.
\end{example}

In Fig.\ref{fig:xyp}, we list the reconstructed images  from two different cases:  $(\rho,\gamma)$ $=$ $(50\%,0\%)$ and $(50\%,30\%)$.  Since the images are too large to show, only their  left corner of size $2048\times2048$ are plotted.  The PSNR and SSIM values are also listed under the corresponding images.  Here we only compare NSS-based QMC  with QMC since both of them are able to avoid the color mixture during the recovering process.  From the numerical results in Fig.\ref{fig:xyp}, it is observed that NSS-based QMC performs better than QMC by producing images with higher PSNR and SSIM values.  From the enlarges parts of size $170\times 170$ in Fig. \ref{fig:xyp_elarge}, one can see that the reconstructions by NSS-based QMC   are smoother than those by QMC and in the case $(\rho,\gamma)=(50\%,30\%)$, there are still some noise left in the recovered image by  QMC and the color edges are not sharp.

\begin{example}[Color Video Inpainting]\label{e:ex3}
In this experiment, we apply the TNSS-based QMC proposed in Section \ref{sec:large} to  the inpainting problem of  color videos.  The testing video  is downloaded from the  `videoSegmentationData' database\footnote{The website of `videoSegmentationData': http://www.kecl.ntt.co.jp/people/kimura.akisato/saliency3.html} in \cite{fmkty09} and contains $72$ frames which are of size $288\times 352$.
We randomly take  20 frames of the color video `M07\_058' and store them in a third-order quaternion tensor.
The observed color video is generated by  randomly removing  $80\%$ pixels from the original color video.
The size searching window $D=20$ and  the number of patches in each patching group $d=60$.  The stopping criteria of the alternating direction method of multipliers are that the norm of the successive iterates is less than $1e$-$4$ and the maximum number of iteration is $100$.
For comparison, we also apply  QMC  to each frame and reconstruct all frames of  color video one by one.
\end{example}

For indication, the first $4$ frames of  original, observed and reconstructed color videos are listed in Fig.\ref{fig:video} and  their enlarged parts of size $60\times 60$ are also shown in Fig.\ref{fig:video}.
 The  PSNR and SSIM values of color videos  are given in Table \ref{tab:video}.
   From these numerical results, we observe that  TNSS-based QMC  successfully reconstructs color video from  20\% known  pixels. It performs better than QMC by producing each frame with higher PSNR and SSIM values.
     Moreover, QMC fails to recover the edges of airplane, but TNSS-based QMC successfully  reconstructs the color information and geometric shape of airplane.  From the enlarged parts  in Fig.\ref{fig:video},  one can observe that  the aircraft nose reconstructed  by TNSS-based QMC  is  smooth and has sharp edge, but  the aircraft nose recovered by  QMC is rough and has blur edge.  The computational cost of CPU time by TNSS-based QMC is about $5.23\%$ of that costed by QMC.

\begin{table}
\caption{
PSNR and SSIM values of the reconstructed videos by QMC and TNSS-based QMC.}
\label{tab:video}
\begin{center}
\renewcommand{\arraystretch}{1.0} \vskip-3mm {\fontsize{8pt}{\baselineskip}%
\selectfont
\scalebox{0.99}{
\begin{tabular}{|c|cc|cc||c|cc|cc|}
\hline
Number of        & \multicolumn{2}{c|}{ QMC } &\multicolumn{2}{c||}{ TNSS-based QMC}& Number of        & \multicolumn{2}{c|}{ QMC } &\multicolumn{2}{c|}{TNSS-based QMC}  \\
 frames      &  PSNR&SSIM &PSNR&SSIM& frames      &  PSNR&SSIM &PSNR&SSIM \\ \hline\hline
1
&  30.33 & 0.9782   &  {\bf  35.83}  & {\bf  0.9936}&
11
&  29.39 & 0.9760   &   {\bf 33.89}& {\bf  0.9926}  \\ \hline
2
&  30.32 & 0.9782   &   {\bf 35.87}  & {\bf  0.9936} &
12
&  29.38 & 0.9759   &   {\bf 33.79}  &  {\bf 0.9924}  \\ \hline
3
&  30.42 & 0.9786   &   {\bf 36.17}  & {\bf 0.9939} &
13
&  29.23 & 0.9752   &   {\bf 33.71} &  {\bf 0.9922}  \\ \hline
4
&30.70 & 0.9796  &  {\bf 35.46}  & {\bf 0.9935} &
14
&  29.86 & 0.9780   &   {\bf 34.21}  &  {\bf 0.9931} \\ \hline
5
&  30.77 & 0.9797   &   {\bf 35.52} &  {\bf 0.9935} &
15
&  30.09 & 0.9789   &   {\bf 33.94}  &  {\bf 0.9925}  \\ \hline
6
&  30.53 & 0.9794   &   {\bf 35.36} &  {\bf 0.9935} &
16
&29.84 & 0.9783   &   {\bf 34.06}  &  {\bf 0.9928} \\ \hline
7
&  30.52 &0.9794   &   {\bf 35.57}  &  {\bf 0.9937} &
17
&  29.82 & 0.9782   &   {\bf 34.07} &  {\bf  0.9928} \\ \hline
8
&  30.47 & 0.9793   &  {\bf  35.54} &  {\bf 0.9938} &
18
&  29.77 & 0.9776   &{\bf 34.91}  & {\bf  0.9938} \\ \hline
9
&  30.11 & 0.9783   &  {\bf  34.15} &  {\bf 0.9928}&
19
&  29.67 & 0.9780   &   {\bf 35.12}  & {\bf  0.9943}  \\ \hline
10
&  29.75 & 0.9775   &  {\bf 34.06}  & {\bf 0.9926} &
20
&  29.56 & 0.9781   &  {\bf 33.97}  & {\bf  0.9928}  \\ \hline
\end{tabular}}
} \vskip6mm
\end{center}
\end{table}

\begin{figure}[!h]
\centering
\includegraphics[width=0.99\linewidth,height=0.60\linewidth]{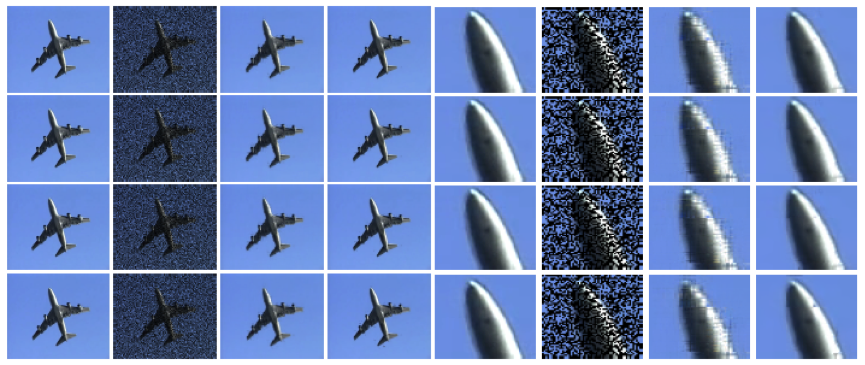}

\caption{First 4 frames and enlarged parts of the original color video (1st column), the observed color video (2nd column),  the reconstructed color video by QMC (3rd column) and the reconstructed color video by TNSS-based QMC (4th column).}
\label{fig:video}
\end{figure}

\begin{table}
\caption{
PSNR and SSIM values of the reconstructed videos by QMC, TNN, TTNN, and TNSS-based QMC.}
\label{tab:videohorsenois01}
\begin{center}
\renewcommand{\arraystretch}{1.0} \vskip-3mm {\fontsize{8pt}{\baselineskip}%
\selectfont
\scalebox{0.99}{
\begin{tabular}{|c|c|cc|cc|cc|cc|cc|}
\hline
Number of   &Type of      & \multicolumn{2}{c|}{ QMC } &\multicolumn{2}{c|}{ TNN}&\multicolumn{2}{c|}{ TTNN}         &\multicolumn{2}{c|}{TNSS-based QMC} \\
 frames   &sparse noise   &  PSNR&SSIM &PSNR&SSIM&  PSNR&SSIM     &PSNR&SSIM \\ \hline\hline
1
& `TUBE' & \underline{34.63}   &  0.9412   &   31.76  &  \underline{0.9446} &   29.88 & 0.8416     &   \bf{40.57} &  \bf{0.9784}\\
& `NON-TUBE' & \underline{34.57}  &  0.9410   &   31.72  &  \underline{0.9458} &   31.47 & 0.8484     &   \bf{40.46}  &  \bf{0.9789}\\ \hline
2
& `TUBE' & \underline{34.78} & 0.9445   &   32.16  &  \underline{0.9497} &   30.35 & 0.8568    &   \bf{41.70} &  \bf{0.9832} \\
& `NON-TUBE' & \underline{34.77}  &  0.9440   &   32.07  &  \underline{0.9497} &   31.85 & 0.8576     &   \bf{41.65}  &  \bf{0.9835}\\ \hline
3
& `TUBE' & \underline{35.00} & 0.9464  &   32.59  & \underline{0.9528} &    30.79 & 0.8696   &   \bf{42.24} &  \bf{0.9850}\\
& `NON-TUBE' & \underline{34.98}  &  0.9457   &   32.44  & \underline{0.9514} &    32.22 & 0.8655     &   \bf{42.42}  &  \bf{0.9856}\\ \hline
4
& `TUBE' & \underline{34.94} & 0.9475  &  32.64  & \underline{0.9524} &    30.90 & 0.8766     &   \bf{42.35} &  \bf{0.9859}\\
& `NON-TUBE' & \underline{35.10}  &  0.9477   &   32.63  & \underline{0.9529} &    32.20 & 0.8689     &   \bf{42.52}  &  \bf{0.9864}\\ \hline
5
& `TUBE' & \underline{35.09} & 0.9474   &   32.61 &  \underline{0.9512} &   30.76 & 0.8773     &   \bf{42.28} &  \bf{0.9858}\\
& `NON-TUBE' & \underline{34.92}  &  0.9463   &   32.59 &  \underline{0.9506} &   32.07 & 0.8659     &   \bf{42.36} &  \bf{0.9863}\\ \hline
6
& TUBE& \underline{34.88} & 0.9452  &   32.64 &  \underline{0.9511}  &  30.85 & 0.8839    &   \bf{41.91} &  \bf{0.9841}\\
&NON-TUBE & \underline{34.86}  &  0.9455   &   32.58 &  \underline{0.9520}  &  32.04 & 0.8667     &   \bf{42.30}  &  \bf{0.9855}\\ \hline
7
& `TUBE' & \underline{34.82} & 0.9443   &   32.55 &  \underline{0.9484}  &  30.93 & 0.8883     &   \bf{41.65} &  \bf{0.9829}\\
& `NON-TUBE' & \underline{34.84}  &  0.9447   &   32.58 &  \underline{0.9504} &  31.95 & 0.8637     &   \bf{41.94}  &  \bf{0.9840}\\ \hline
8
& `TUBE' & \underline{35.02} & 0.9456  &   32.62 &  \underline{0.9505}  &  31.04 & 0.8931     &   \bf{42.16} &  \bf{0.9849}\\
& `NON-TUBE' & \underline{35.02}  &  0.9457   &  32.64&  \underline{0.9512}  &  32.00 & 0.8676     &   \bf{42.35}  &  \bf{0.9861}\\ \hline
9
& `TUBE' & \underline{34.93} & 0.9454  &   32.62 &  \underline{0.9505}  &  30.92 & 0.8910     &   \bf{42.16} &  \bf{0.9854}\\
&NON-TUBE & \underline{34.91}  &  0.9451   &   32.64 &  \underline{0.9517}  &  32.03 & 0.8633     &   \bf{42.54}  &  \bf{0.9866}\\ \hline
10
& `TUBE' & \underline{34.94} & 0.9456   &   32.47 &  \underline{0.9503}  &  30.65 & 0.8882     &   \bf{42.04} &  \bf{0.9850}\\
& `NON-TUBE' & \underline{34.99}  &  0.9457   &   32.64 &  \underline{0.9515}  &  31.84 & 0.8596     &   \bf{42.38}  &  \bf{0.9861}\\ \hline
\end{tabular}}
} \vskip6mm
\end{center}
\end{table}

\begin{table}
\caption{
PSNR and SSIM values of the reconstructed videos by QMC, TMac$\_{}$TT, LRQTC, and TNSS-based QMC.}
\label{tab:videohorsennan08}
\begin{center}
\renewcommand{\arraystretch}{1.0} \vskip-3mm {\fontsize{8pt}{\baselineskip}%
\selectfont
\scalebox{0.99}{
\begin{tabular}{|c|c|cc|cc|cc|cc|cc|}
\hline
Number of    &    Type of & \multicolumn{2}{c|}{ QMC } &\multicolumn{2}{c|}{ TMac$\_{}$TT} &\multicolumn{2}{c|}{ LRQTC}& \multicolumn{2}{c|}{TNSS-based QMC} \\
 frames      & pixel missing&  PSNR&SSIM &PSNR&SSIM&  PSNR&SSIM &  PSNR&SSIM\\ \hline\hline
1
& `TUBE' & 30.32 & 0.8445   &   30.71             &  0.8274    &\underline{32.47} &\underline{0.8900}        &\bf{34.62} &\bf{0.9324}\\
& `NON-TUBE' &  30.25     &   0.8448  &   31.53      &0.8514      &\underline{34.15}&\underline{0.9263}         &\bf{35.33} &\bf{0.9424}    \\ \hline
2
& `TUBE' & 30.37 & 0.8473     & 30.90  &  0.8315    &\underline{32.60}&\underline{0.8933}           & \bf{34.68} &  \bf{0.9363} \\
& `NON-TUBE' & 29.94 &  0.8425   &  31.89        &  0.8569  &\underline{34.31}&\underline{0.9245}      & \bf{34.95}  & \bf{0.9436}  \\ \hline
3
& `TUBE' & 30.46 & 0.8492  &   31.13  & 0.8357 & \underline{32.64}&\underline{0.8952}&      \bf{35.07 }&  \bf{0.9409}\\
& `NON-TUBE' & 30.49  & 0.8487    & 32.23     &     0.8650    &\underline{34.67}&\underline{0.9284}   &  \bf{35.71}  & \bf{0.9473}  \\ \hline
4
& `TUBE' & 30.66 & 0.8535  &  31.14  & 0.8372 & \underline{32.82}&\underline{0.8978}&      \bf{34.69 }&  \bf{0.9390}\\
& `NON-TUBE' &  30.68  &  0.8520    &  32.27      &  0.8679  &\underline{35.04}&\underline{0.9341}   &  \bf{36.04}   & \bf{0.9494}  \\ \hline
5
& `TUBE' & 30.73 & 0.8553   &   31.01 &  0.8353 & \underline{32.69}&\underline{0.8964}&     \bf{34.79} &  \bf{0.9407}\\
& `NON-TUBE' & 30.38   &  0.8490    &  32.05    &   0.8655   &\underline{34.82}&\underline{0.9346}   &  \bf{36.31}    &  \bf{0.9525}  \\ \hline
6
& `TUBE' & 30.70 & 0.8553  &   30.97 &  0.8337 & \underline{32.60}&\underline{0.8937}&   \bf{34.53} &  \bf{0.9390}\\
& `NON-TUBE' & 30.47  &  0.8513    &  32.17     &   0.8677 &\underline{34.88}&\underline{0.9337}    &   \bf{35.65}   &  \bf{0.9481}  \\ \hline
7
& `TUBE' & 30.64 & 0.8521   &   30.90 &  0.8314  & \underline{32.57}&\underline{0.8903}&    \bf{34.50} &  \bf{0.9356}\\
& `NON-TUBE' & 30.42   & 0.8486     &  32.03    &   0.8648 &\underline{34.73}&\underline{0.9321}     &  \bf{35.68}   &  \bf{0.9474}  \\ \hline
8
& `TUBE' & 30.66 & 0.8519  &   30.99 &  0.8343  & \underline{32.61}&\underline{0.8896}&   \bf{34.52} &  \bf{0.9341}\\
& `NON-TUBE' & 30.79  & 0.8540    &   32.05    &    0.8667 &\underline{34.70}&\underline{0.9311}    &   \bf{36.10}  &   \bf{0.9498} \\ \hline
9
& `TUBE' & 30.42 & 0.8503  &   30.90 &  0.8313  & \underline{32.65}&\underline{0.8905}&  \bf{34.10} &  \bf{0.9339}\\
& `NON-TUBE' & 30.70  & 0.8506   &  31.98    &   0.8625 &\underline{34.93}&\underline{0.9319}     &   \bf{36.13}  &   \bf{0.9521} \\ \hline
10
& `TUBE' & 30.46 & 0.8514   &   30.79 &  0.8307 &\underline{32.58}&\underline{0.8901}  &   \bf{34.15} &  \bf{0.9344}\\
& `NON-TUBE' & 30.54  &  0.8518    &  31.78    &    0.8596  &\underline{34.51}&\underline{0.9288}   &  \bf{35.85}   &  \bf{0.9500}  \\ \hline
\end{tabular}}
} \vskip6mm
\end{center}
\end{table}

\begin{figure}
\centering
\includegraphics[width=0.99\linewidth,height=0.95\linewidth]{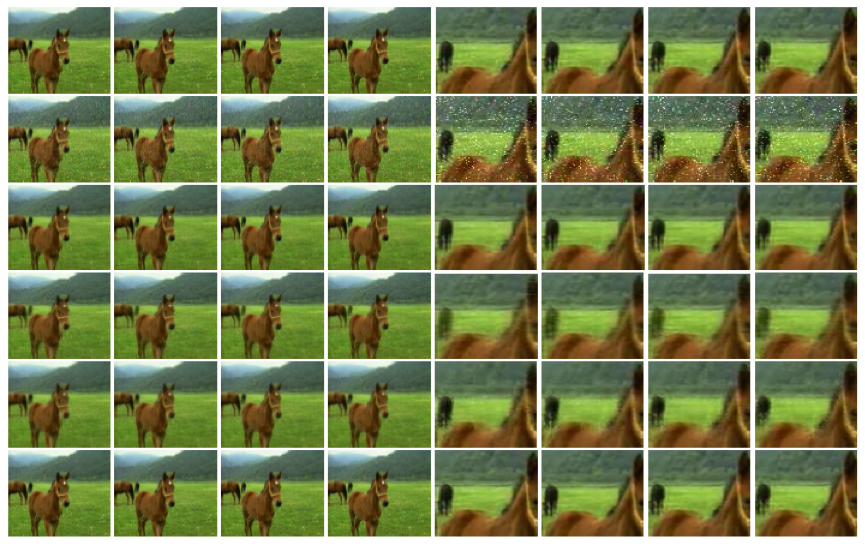}
\caption{First 4 frames and enlarged parts of the original color video (1st row), the observed color video (2nd row),  and the reconstructions by QMC (3rd row), TNN (4th row),  TTNN (5th row) and  TNSS-based QMC (last row). }
\label{fig:videohorsenois01}
\end{figure}

\begin{figure}
\centering
\includegraphics[width=0.99\linewidth,height=0.95\linewidth]{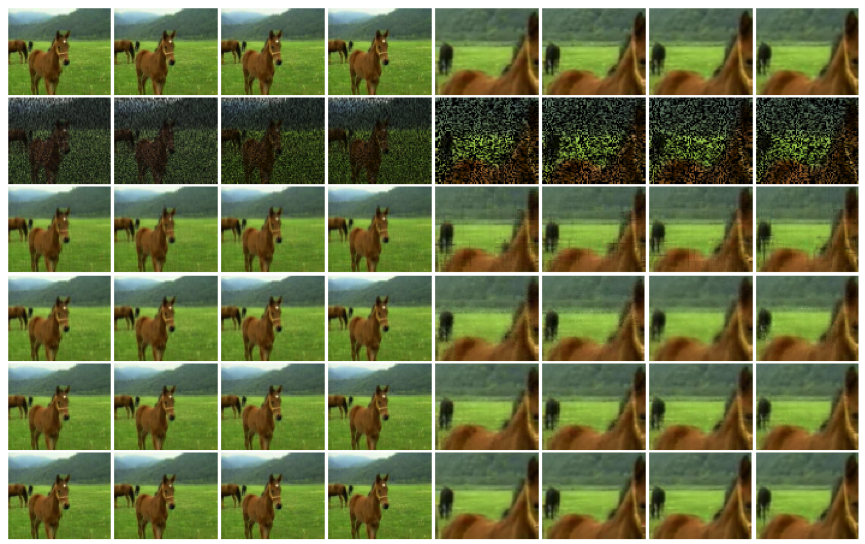}
\caption{First 4 frames and enlarged parts of the original color video (1st row), the observed color video (2nd row),  and the reconstructions by QMC (3rd row),  TMac$\_{}$TT (4th row),  LRQTC (5th row) and TNSS-based QMC (last row).}
\label{fig:videohorsenan08random}
\end{figure}

\begin{example}[Comparison with tensor completion methods
] \label{e:ex4}
In this experiment, we compare  TNSS-based QMC
 with  QMC,  TMac$\_{}$TT and three other tensor completion methods:
\begin{itemize}
\item  TNN--A method based on the tensor tubal rank \cite{CY2016}.
\item  TTNN--A model for tensor robust principle component analysis based on tensor train rank \cite{yang2020}.
\item  LRQTC--Low-rank quaternion tensor completion method \cite{miao2020}.
\end{itemize}
We implement the comparison on a widely used  color video `DO01$\_{}$013'  from the  `videoSegmentationData' database \cite{fmkty09}, which contains $89$ frames  of size $288\times 352$.  According to the setting mentioned in \cite{yang2020}, we need to resize the testing color video  into $243\times 256\times 3 \times 27$  
to implement TTNN and TMac$\_{}$TT. The frame mode is merged with the image row mode to form a third-order tensor, called by a video sequence tensor, of size $6561\times 256\times 3$. 
  Then the tensor is reshaped into a ninth-order tensor of size $6 \times 6\times 6\times 6\times 6\times 6\times 6\times 6\times 3$ using the KA technique \cite{j2005}. The ninth-order video sequence tensor  is directly used for the tensor completion algorithms (TTNN and TMac$\_{}$TT).

For the validity of comparison, we set two cases:
\begin{itemize}
\item  {\bf Case 1}: The observed color video is generated by randomly adding sparse noise to $10\%$ pixels of the original color video.
\item  {\bf Case 2}:  The observed color video is generated by  randomly removing  $80\%$ pixels from the original color video.
\end{itemize}
TNN and TTNN are tested in Case 1, since they can't  handle the problem with missing pixels.
TMac$\_{}$TT and LRQTC are tested in Case 2, because they can't handle the problem with sparse noise.
QMC and TNSS-based QMC are tested in two cases. The size of searching window $D=20$ and  the number of patches in each patching group $d=60$.
The stopping criteria of the alternating direction method of multipliers are that the norm of the successive iterates is less than $1e$-$4$ and the maximum number of iteration is $100$.
Moreover, two types of sparse noise and pixel missing are designed. One type is called  `TUBE' missing or noise, where all frames are missing or noisy at the same positions.  Another type is called `NON-TUBE' missing or noise, where all frames are missing or noisy at different locations but at the same percentage.

\end{example}

For Case 1,  the  PSNR and SSIM values  are given in Table \ref{tab:videohorsenois01}.
 The numerical results indicate that  TNSS-based QMC  is superior to other compared methods.
For instance,  within `TUBE' noise,  TNSS-based QMC has the average PSNR and SSIM values: $(41.91,~0.9841)$, which are higher than QMC $(34.90,~0.9454)$, TNN $(32.47,~0.9501)$, and TTNN $(30.71,~0.8766)$.
The first $4$ frames of original, observed  (with `TUBE' noise), reconstructed  color videos and their enlarged parts are shown in Fig.\ref{fig:videohorsenois01}. 
It can be observed that all compared methods remove sparse noise well,  and the  horse silhouette and mane recovered by  TNSS-based QMC are more obvious than the other methods.

 For  Case 2,  the  PSNR and SSIM values   are given Table \ref{tab:videohorsennan08}.
  The numerical results also indicate that  TNSS-based QMC  is superior to other compared methods (LRQTC is the second best one).
For instance, within `NON-TUBE' missing (see Table \ref{tab:videohorsennan08}), TNSS-based QMC has the average PSNR and SSIM values: $(35.77,~0.9483)$, which are higher than QMC $(30.47,~0.8493)$,  TMac$\_{}$TT $(32.00,~0.8596)$, and LRQTC $(34.67,~0.9300)$.
The first $4$ frames of original, observed (with `NON-TUBE' missing), reconstructed color videos and their enlarged parts are shown in Fig.\ref{fig:videohorsenan08random}. 
We can also observe that all compared methods successfully recover the missed pixels,  and visually,    TNSS-based QMC reconstruct the  horse silhouette and mane at a much higher level than the other methods.  Interestingly, TMac$\_{}$TT, LRQTC and NSS-based QMC perform better in the situation of `NON-TUBE'  than in the situation of `TUBE'.   This numerically confirms our previous judgement that the inpainting problem with incomplete and corrupted  tubes is one of the challenging problems.   Fortunately, the proposed TNSS-based QMC is good at handling such color video inpainting problem with `TUBE' missing.  For instance, TNSS-based QMC improves the average SSIM value from the second highest value  0.9027 (LRQTC) to 0.9456   in the `TUBE' type, while it increases    the average SSIM value from the second highest value $0.9305$ (LRQTC) to $0.9483$ in the `NON-TUBE' type.

Combined with Example
\ref{e:ex3}, the above numerical results show that TNSS-based QMC is able to reconstruct color video from incomplete and/or corruptted pixels
 and its performance is superior to the state-of-the-art methods  in terms of PSNR and SSIM values. More valuable is that the TNSS-based QMC can effectively recover the color information and geometry of objects in color videos, thus presenting  visually comfortable reconstructions.

\section{Conclusion}\label{sec:con}
In this paper, a theoretical analysis is proposed for the working principle of NSS-based approaches.
 Based on this, a new NSS-based QMC method is presented and applied to solve the color image inpainting problem.    The NSS prior is applied to search similar patches,  gathered into a low-rank  quaternion matrix, of a color image, and a good reconstruction is  computed  by a new QMC algorithm.  A novel TNSS-based QMC method  is also developed to inpaint color videos from  incomplete and corrupted frontal tubes. Numerical experiments on (large-scale) color image and video inpainting indicate the superiority  of the newly proposed methods,  resulting in high PSNR and SSIM measures and particularly the best visual quality among competing methods.
If all entries are real numbers, then the TNSS-based QMC method reduces to the TNSS-based matrix completion method, which can be applied to solve the grey video inpainting problem.

In our numerical experiments, the proposed TNSS-based QMC has superiority on color video inpainting to the state-of-the-art methods.
This has been verified within the  `TUBE' and `NON-TUBE'  types of  incompleteness and corruption.
We find the `TUBE' type of incompleteness more challenging. Fortunately, TNSS-based QMC can solve this problem well.
 However,  there is still an open problem that how to inpaint a color video from a piece of incompleteness and corruption. This will be a topic worth studying in the future.
  New fast algorithms for the proposed methods also need to be further investigated.  We refer to  \cite{yuan2009,lazendic2021} as preliminary knowledge for future research.


%


\section*{Acknowledgment}
This work is  supported  in part by the National Natural Science Foundation of China under grants 12171210, 12090011,  11771188, 61876203, 12171072 and 12061052;  HKRGC Grant Numbers: GRF
12200317, 12300218, 12300519 and 17201020;
Natural Science Fund of Inner Mongolia Autonomous Region (No. 2020MS01002);
the Major Projects of Universities in Jiangsu Province (No. 21KJA110001);  and the Applied Basic 
Research Project of Sichuan Province (No. 2021YJ0107); the Priority Academic Program Development Project (PAPD);  and  the Top-notch Academic Programs Project (No. PPZY2015A013) of Jiangsu Higher Education Institutions.

We are grateful to the handing editor and the anonymous referees  for their useful suggestions and to Dr. Yongyong Chen and Dr. Jifei Miao for sharing  with us their MATLAB codes.

\bibliographystyle{siam}

\appendix
\section{\bf The supplementary material of Section \ref{ss:admm}}
In this section we present the detailed steps of solving the joint minimization problem (14) in the main body.

The $[\bf{L},\bf{Q}]$ subproblem of (14) is 
\begin{equation*} \label{e:LQsubproblem}
\begin{array}{rl}
\underset{{\bf L}, {\bf Q}}{\min}\!&\!
\left\Vert {\bf L} \right \Vert_{\ast}\!\! +\frac{\mu}{2} \left\Vert {\bf L-\bf P+\bf Y/\mu}\right\Vert _{F}^2+\frac{\mu}{2} \left\Vert {\bf S-\bf Q+\bf Z/\mu}\right\Vert _{F}^2\\
\ {\rm s.t.} &
\mathcal{P}_{\Omega}\left({\bf P}+{\bf Q}\right)={\bf X},
\end{array}
\end{equation*}
where  $\mu$ is the penalty parameter.
Since $\bf{L}$ and $\bf{Q}$ are decoupled, the $[\bf{L},\bf{Q}]$ subproblem can be solved separately. The $\bf{L}$  subproblem is
\begin{equation} \label{lsub}
\begin{array}{rl}
\underset{{\bf L}}{\min}&
\left\Vert {\bf L} \right \Vert_{\ast} +\frac{\mu}{2} \left\Vert {\bf L-\bf P+\bf Y/\mu}\right\Vert _{F}^2,
\end{array}
\end{equation}
which has the exact solution as
$
{\bf U}\Sigma_{\frac{1}{\mu}}{\bf V}^H,
$
where ${\bf  P-\bf Y/\mu} = {\bf U}  \Sigma {\bf V}^H$ and $ \Sigma_{\frac{1}{\mu}}=\diag(\{\max(\sigma_i-\frac{1}{\mu},0)\}) $.  The computational cost of updating $\bf{L}$ is $O(\min (n_1n_2^2,n_1^2 n_2))$.
The $\bf{Q}$ subproblem
\begin{equation} \label{qcmodel}
\underset{\mathcal{P}_{\Omega}\left({\bf P}+{\bf Q}\right)\!=\!{\bf X}}{\min}\;\;\;  \left\Vert {\bf S-\bf Q+\bf Z/\mu}\right\Vert _{F}^2
\end{equation}
can be exactly solved as
$$
{\bf Q}_{i,j} =
\left\{ {\begin{array}{*{20}{l}}
({\bf  X}-{\bf  P})_{i,j}, &\text{if}~ (i,j)\in \Omega,\\
({\bf S}+{\bf Z}/\mu)_{i,j},  & \text{if}~ (i,j)\in \bar{\Omega}.
\end{array}} \right.$$  
The computational cost of updating  $\bf{Q}$ is $O(n_1n_2)$.

The $[\bf{S},\bf{P}]$ subproblem of (14) is 
\begin{equation*} \label{e:SQsubproblem}
\begin{array}{rl}
\underset{{\bf S}, {\bf P}}{\min}&
 \lambda \left\Vert {\bf S}\right\Vert _{1}+\frac{\mu}{2} \left\Vert {\bf L-\bf P+\bf Y/\mu}\right\Vert _{F}^2+\frac{\mu}{2} \left\Vert {\bf S-\bf Q+\bf Z/\mu}\right\Vert _{F}^2\\
\ {\rm s.t.} &
\mathcal{P}_{\Omega}\left({\bf P}+{\bf Q}\right)={\bf X},
\end{array}
\end{equation*}
where  $\mu$ is the penalty parameter.
Since $\bf{S}$ and $\bf{P}$ are decoupled, the $[\bf{S},\bf{P}]$ subproblem can solved separately.
The $\bf{S}$  subproblem is
\begin{equation} \label{scmodel}
\begin{array}{rl}
\underset{{\bf S}}{\min}&
\lambda \left\Vert {\bf S}\right\Vert _{1}+\frac{\mu}{2} \left\Vert {\bf S-\bf Q+\bf Z/\mu}\right\Vert _{F}^2,
\end{array}
\end{equation}
which has the exact solution $
{\text{soft}}_{\frac{\lambda}{\mu}}\left( \bf Q-\bf Z/\mu\right)
$ via the component-wise soft thresholding
\begin{equation*}
\text{soft}_{\frac{\lambda}{\mu}}(\bf y) = \mathrm{sign}(\bf y) \max(|\bf y| - \frac{\lambda}{\mu},0).
\end{equation*}
 The computational cost of updating ${\bf S}$ is $O(n_1n_2)$.
The $\bf{P}$ subproblem
\begin{equation} \label{pqcmodel}
\underset{\mathcal{P}_{\Omega}\left({\bf P}+{\bf Q}\right)\!=\!{\bf X}}{\min}\;\left\Vert {\bf L-\bf P+\bf Y/\mu}\right\Vert _{F}^2
\end{equation}
can be exactly solved as
$$
{\bf P}_{i,j} =
\left\{ {\begin{array}{*{20}{l}}
{\bf F}_{i,j}/2\mu, &\text{if}~ (i,j)\in \Omega,\\
({\bf L}+{\bf Y}/\mu)_{i,j},  & \text{if}~ (i,j)\in \bar{\Omega},
\end{array}} \right.$$
where $\Fq=\mu {\bf L}+\mu {\bf X}- \mu {\bf S}+{\bf Y}- {\bf Z}$.
  The computational cost of updating $\bf{P}$ is $O(n_1n_2)$.

\section{The supplementary material of Example \ref{e:ex1}}
In this section,   we compare the efficiency of NSS-based QMC with various sizes of patch and searching window.
Without changing the setting  above experiments, two cases are tested:
\begin{itemize}
\item
{\bf Case 1:} The size of patch is set as $s_1=6$ and the size of searching window  is randomly chosen as about multiplications of $s_1$,  $s_2=12,18,24,30,54$.
\item
{\bf Case 2:} The size of patch is set as $s_1=16$ and the size of searching window  is randomly chosen as  $s_2=16,18,20,32,64$.
\end{itemize}
The reconstructed color images are listed in Fig.\ref{fig:flower2}, and the variation trend of PSNR and SSIM values are shown in Figs.\ref{fig:flower3_1}--\ref{fig:flower3_2}.
From these numerical results, we find that  the smaller the patch size, the better the reconstruction. As the searching window increases, the PSNR and SSIM values firstly go up and then go down. A proper size of searching window depends on the data.

\begin{figure}[!h]
\centering
\includegraphics[width=0.9\linewidth,height=0.8\linewidth]{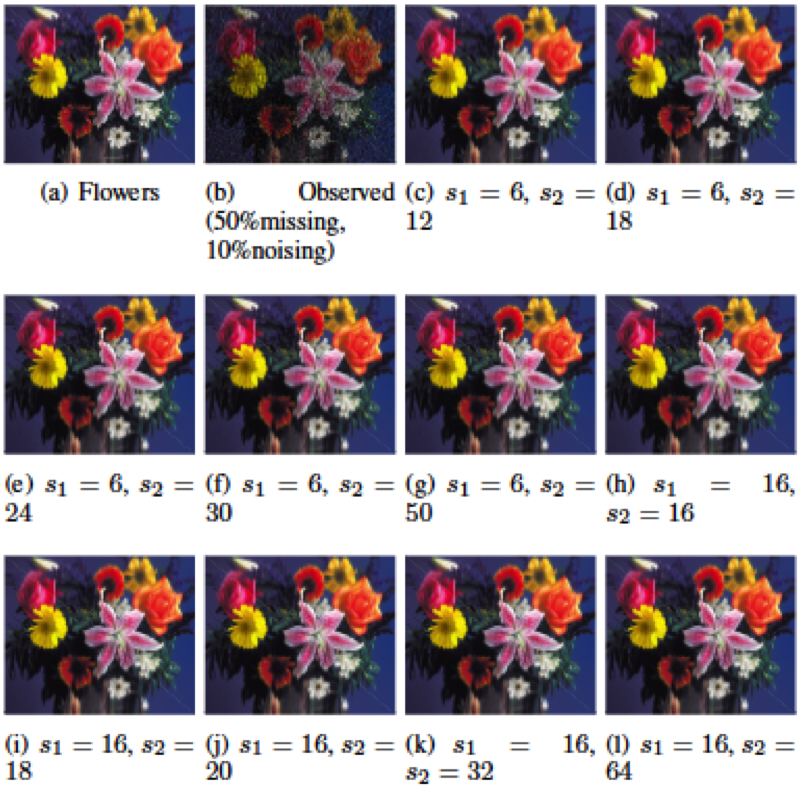}
\caption{Visual comparison of completion results by NSS-based QMC with different sizes of patch ($s_1$) and searching window ($s_2$).}
\label{fig:flower2}
\end{figure}

\begin{figure}[!h]
\centering
\includegraphics[width=0.9\linewidth,height=0.8\linewidth]{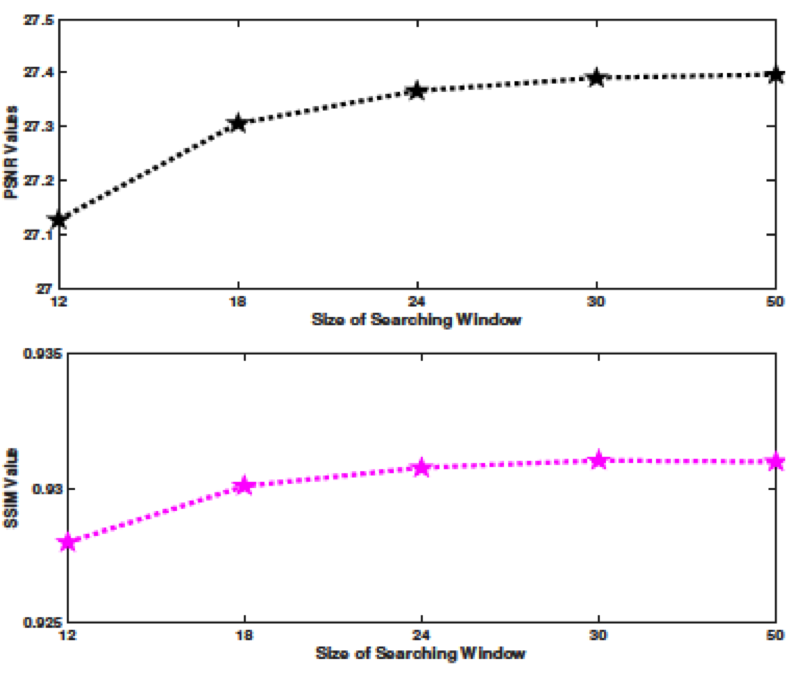}
\caption{PSNR and SSIM Values  of completion results by NSS-based QMC with  $s_1=6$ and $s_2=12,18,24,30,50$. }
\label{fig:flower3_1}
\end{figure}

\begin{figure}[!h]
\centering
\includegraphics[width=0.9\linewidth,height=0.8\linewidth]{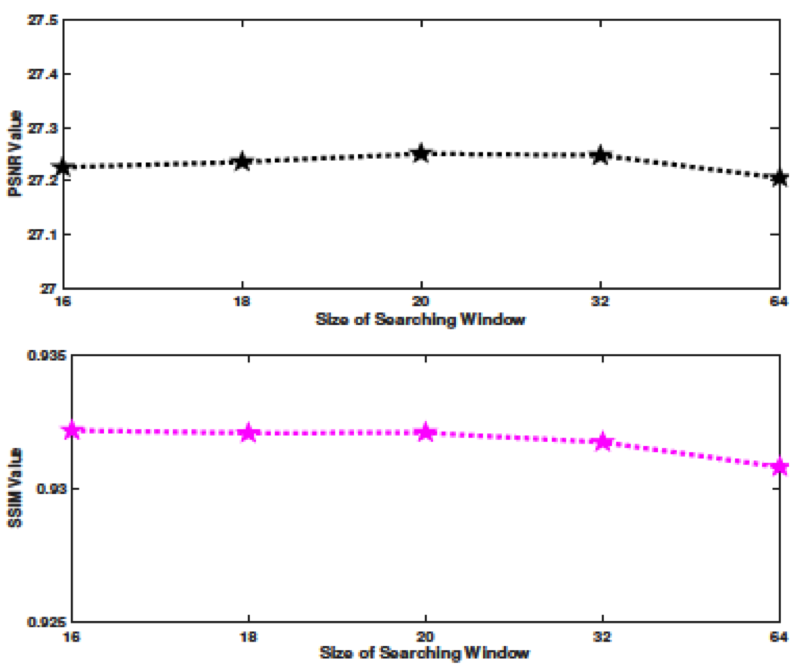}
\caption{PSNR and SSIM Values  of completion results by NSS-based QMC with  $s_1=16$ and $s_2=16,18,20,32,64$. }
\label{fig:flower3_2}
\end{figure}

\newpage

\section{\bf The supplementary material of  Example \ref{e:ex4}}
In this section, we show first 4 frames of original, observed and recovered video in Case 1 that  10$\%$ percent pixels are corrupted by  sparse noise   and in the `NON-TUBE' type of corruption and their corresponding enlarged parts in Fig.\ref{fig:videohorsenois01random}.
\begin{figure}
\centering
\includegraphics[width=0.9\linewidth,height=0.6\linewidth]{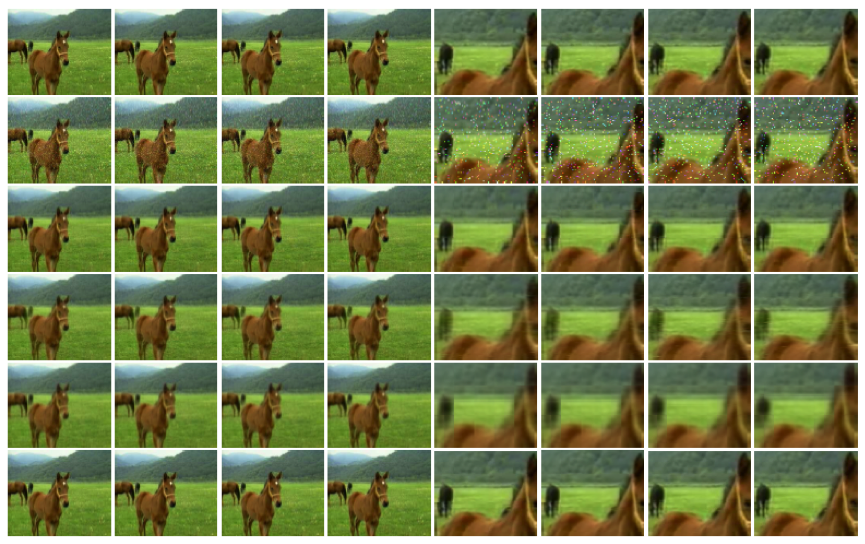}
\caption{First 4 frames and enlarged parts of the original color video (1st row), the observations (2nd row) and the reconstructions by QMC (3rd row),  TNN (4th row),  TTNN (5th row), and  TNSS-based QMC (last row).}
\label{fig:videohorsenois01random}
\end{figure}

Next, we show show first 4 frames of original, observed and recovered videos in Case 2 that 80$\%$ percent pixels are missing randomly and in the  `TUBE' type of incompleteness and their corresponding enlarged parts in Fig.\ref{fig:video2horse_TMac}.

\begin{figure}
\centering
\includegraphics[width=0.9\linewidth,height=0.6\linewidth]{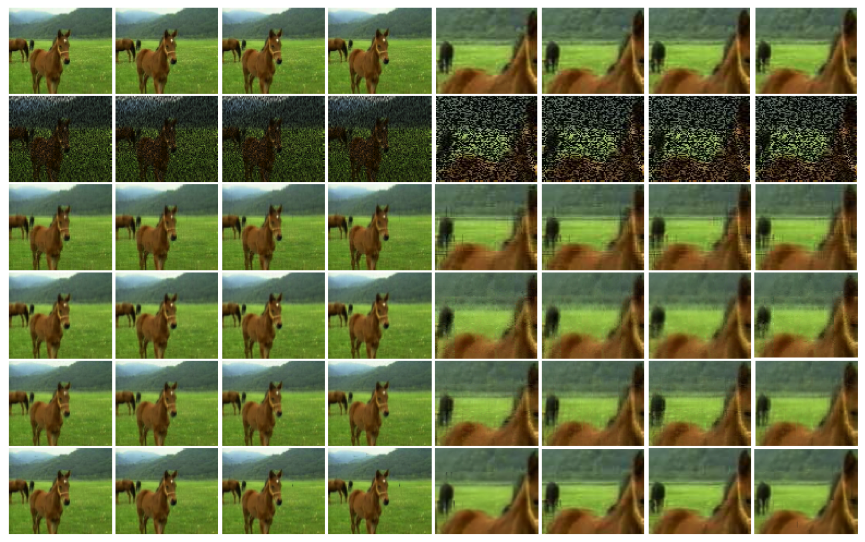}
\caption{First 4 frames  and enlarged parts of the original color video (1st row), the observe color video(2nd row) and the reconstructions by  QMC(3rd row),  TMac$\_{}$TT(4th row), and TNSS-based QMC(last row).}
\label{fig:video2horse_TMac}
\end{figure}

\section{More numerical and visual results}

 In this section, we present more numerical and visual results for Example \ref{e:ex1} in the main body.
  Several comparisons were implemented on the additional seven standard color images under the situation: $(1-\rho,\gamma)=(50\%,10\%)$. The stopping criteria are that the norm of the successive iterates is less than $1e$-$4$ and the maximum number of iteration is $500$.  The numerical results are shown in Table \ref{tab:7im}.   The best PSNR and SSIM values are  shown in bold,  and the second-best ones are underlined.
  The corresponding color images are shown in Fig.\ref{fig:qmc7im}.

From the numerical and visual comparisons, it is not difficult to find that TMac$\_{}$TT [20], LRQA-L [17] and LRQMC [18] fail to handle the robust color image inpainting problem ~(i.e., inpainting color image from incomplete and corrupted pixels), which is consistent with the models in the corresponding papers. What's more, there are three methods, QMC, TC and NSS-based QMC successfully recovering color images from incomplete and corrupted pixels. Especially, both from the PSNR, SSIM values and the recovered images in the situation $(1-\rho,\gamma)=(50\%,10\%)$, the performance of NSS-based QMC is the most excellent among them. Take `Sails' as an example, the PSNR and SSIM values of NSS-based QMC are (26.99, 0.8389), which are higher than those of QMC (23.97, 0.7070) and TC (24.12, 0.7072). For the reconstructed images, the edges recovered by NSS-based QMC are much cleaner than those recovered by two other methods.

\begin{figure}
\centering
\includegraphics[width=0.99\linewidth,height=0.95\linewidth]{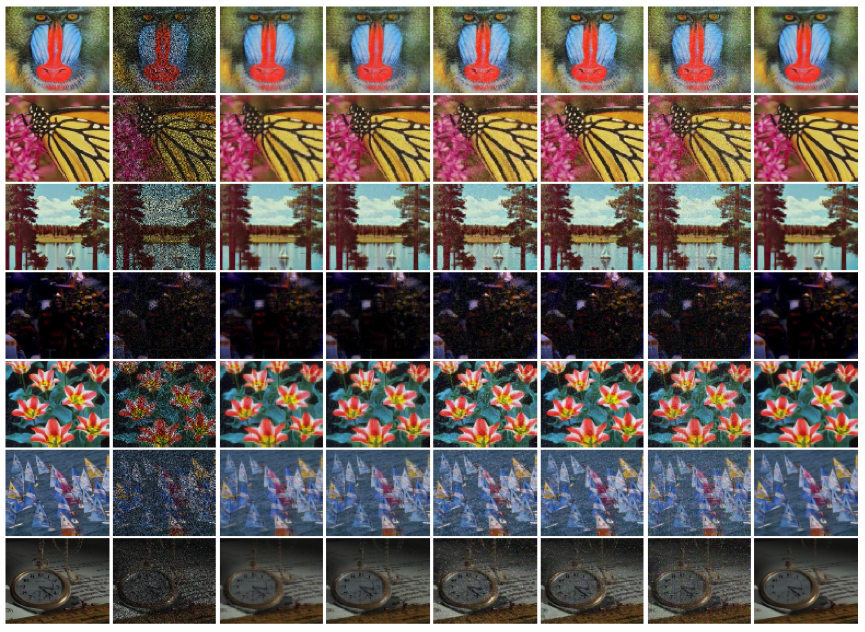}
\caption{The original color images (1st column), the observations (2nd column),  and the reconstructions by QMC (3rd column), TC(4th column),  TMac$\_{}$TT (5th column), LRQA-L(6th column), LRQMC(7th column),  and  TNSS-based QMC (last column). }
\label{fig:qmc7im}
\end{figure}

 \begin{table}
\caption{
PSNR and SSIM values of the reconstructed images by QMC, TC, TMac$\_{}$TT, LRQA-L, LRQMC, and NSS-based QMC.}
\label{tab:7im}
\begin{center}
\renewcommand{\arraystretch}{1.0} \vskip-3mm {\fontsize{8pt}{\baselineskip}%
\selectfont
\scalebox{0.850}{
\begin{tabular}{|c|cc|cc|cc|cc|cc|cc|}
\hline
Images        & \multicolumn{2}{c|}{ QMC } &\multicolumn{2}{c|}{ TC}& \multicolumn{2}{c|}{ TMac$\_{}$TT}& \multicolumn{2}{c|}{ LRQA-L } &\multicolumn{2}{c|}{LRQMC}&\multicolumn{2}{c|}{NSS-based QMC}  \\
       &  PSNR&SSIM &PSNR&SSIM&PSNR&SSIM&PSNR&SSIM      &  PSNR&SSIM &PSNR&SSIM \\ \hline\hline
Baboon
&  22.99 & 0.6523   &  22.88  & \underline{0.6533}&20.01& 0.5617 &  18.71 & 0.4822   & \underline{24.34}  & 0.4290 &  {\bf 24.90 }& {\bf 0.7644}  \\ \hline
Monarch
&  22.58 & \underline{0.7271}&\underline{23.59}  & 0.7084 &19.82&0.5575  & 19.30& 0.5395   & 23.52& 0.3949 &  {\bf 27.17}  &  {\bf 0.8716}  \\ \hline
Sailboat
&  23.39 & \underline{0.7322}   &   23.90  & 0.7163 &20.25  & 0.5487 &  18.90 & 0.4686   &\underline{24.34} & 0.3797 &  {\bf 25.72} &  {\bf 0.9283}  \\ \hline
Carnev
&24.12 & 0.7359  &  \underline{24.90}  & \underline{0.7621} & 18.47  & 0.2728   &  19.00 & 0.3526   &23.01 & 0.2266 &  {\bf 28.27}  &  {\bf 0.9179} \\ \hline
Tulips
&  \underline{23.46} & 0.7409   &   23.36 &  \underline{0.7451} & 19.52 & 0.5177 &  19.19 & 0.5368   &23.43 & 0.4114 &  {\bf 27.96}  &  {\bf 0.8637}  \\ \hline
Sails
&  23.97 & 0.7070   &   \underline{24.12} &  \underline{0.7072} & 20.05 &0.3966  &19.08 & 0.4930   &23.72 & 0.3968 &  {\bf 26.99}  &  {\bf 0.8389} \\ \hline
Watch
&  24.71 &\underline{0.7247}   &   \underline{25.77}  &  0.7168 & 19.29 &0.4315  &  19.29 & 0.4394   &23.99 & 0.3034 &  {\bf 27.78} &  {\bf 0.9057} \\ \hline
\end{tabular}}
} \vskip6mm
\end{center}
\end{table}

\end{document}